\documentclass{article}

\usepackage{hyperref}       
\usepackage[final]{neurips_2019}

\usepackage[utf8]{inputenc} 
\usepackage[T1]{fontenc}    

\usepackage{xcolor} 
\usepackage{dm-colors}
\hypersetup{
	colorlinks,
	citecolor=dmblue300,
	linkcolor=dmred300,
	urlcolor=dmorange500}

\usepackage{booktabs}       
\usepackage{amsfonts}       
\usepackage{nicefrac}       
\usepackage{microtype}      
\usepackage{thm-restate}
\usepackage{amsthm}
\usepackage{amsmath}
\usepackage{amssymb}
\usepackage{bm}
\usepackage{bbm}
\usepackage{mathtools}
\usepackage{xspace}
\usepackage{wrapfig}
\usepackage{bookmark}

\usepackage[disable]{todonotes}

\usepackage{graphicx} 
\usepackage{subcaption}

\usepackage{algorithm}
\usepackage{algpseudocode}

\usepackage{placeins}

\usepackage{tikz}

\usepackage[shortlabels]{enumitem}

\newcommand{\algdelta}{\delta'}

\newcommand{\pacdelta}{\delta}

\newcommand{\auxdelta}{\tilde{\delta}}

\newcommand{\supFsZero}{M_{\lambda}}

\newcommand{\barepsilon}{\kappa}
\newcommand{\cOtilde}[1]{\tilde{\mathcal{O}}\pa{#1}}

\newcommand{\estV}{{\tt sampleV}}
\newcommand{\estQ}{{\tt estimateQ}}
\newcommand{\Vtilde}{\tilde{V}}
\newcommand{\event}{\mathcal{A}}
\newcommand{\subevent}{\mathcal{B}}
\newcommand{\auxevent}{\mathcal{E}}
\newcommand{\params}{\mathbf{params}}

\newcommand{\nsamplev}{n_{\estV}}
\newcommand{\oracle}{{\tt oracle}}
\newcommand{\policy}{\pi}
\newcommand{\entropy}{\mathcal{H}}
\newcommand{\cP}{\mathcal{P}}

\newcommand{\given}{\Big|}
\newcommand{\logsumexp}[1]{ \mathrm{LogSumExp}_{#1} }
\newcommand{\clip}[2]{\mathbf{clip}_{#1}\left(#2\right)}

\newcommand{\overeq}[1]{\stackrel{\mathbf{#1}}{=}}
\newcommand{\overleq}[1]{\stackrel{\mathbf{#1}}{\leq}}

\newcommand{\imag}{\mathrm{i}}

\newcommand{\genericmcts}{{\tt genericMCTS}}
\newcommand{\mctssearch}{{\tt search}}
\newcommand{\evalleaf}{{\tt evaluateLeaf}}
\newcommand{\selectaction}{{\tt selectAction}}

\newcommand{\transprob}{P}
\newcommand{\rewardfunc}{R}
\newcommand{\nextstate}{z}

\newcommand{\states}{\mathcal{S}}
\newcommand{\actions}{\mathcal{A}}
\newcommand{\svar}{s}

\newcommand{\ValueFunc}{V}
\newcommand{\Value}[1]{\ValueFunc\left(#1\right)}

\newcommand{\ValueFuncH}{\hat{V}}
\newcommand{\ValueH}[1]{\ValueFuncH\left(#1\right)}
\newcommand{\QFunc}{Q}
\newcommand{\QFuncHat}{\hat{Q}}
\newcommand{\QFuncBar}{\bar{Q}}

\newcommand{\rewardrand}{R}
\newcommand{\transrand}{Z}

\newcommand{\EE}[2][]{\mathbb{E}_{#1}\left[#2\right]}

\newcommand{\PP}[1]{\mathbb{P}\left[#1\right]}
\newcommand{\ceil}[1]{\left\lceil#1\right\rceil}

\newcommand{\cO}[1]{\mathcal{O}\pa{#1}}

\newcommand{\indic}[1]{\mathbb{I}_{\left\lbrace #1 \right\rbrace}}

\newcommand{\pa}[1]{\left(#1\right)}

\DeclarePairedDelimiterX{\norm}[1]{\lVert}{\rVert}{#1}
\DeclarePairedDelimiterX{\abs}[1]{\lvert}{\rvert}{#1}
\newcommand{\normm}[1]{\lVert#1\rVert} 

\newcommand{\transpose}{^\mathsf{\scriptscriptstyle T}}

\newcommand{\eps}{\varepsilon}
\renewcommand{\epsilon}{\varepsilon}
\renewcommand{\hat}{\widehat}
\renewcommand{\tilde}{\widetilde}
\renewcommand{\bar}{\overline}

\let\originalleft\left
\let\originalright\right
\renewcommand{\left}{\mathopen{}\mathclose\bgroup\originalleft}
\renewcommand{\right}{\aftergroup\egroup\originalright}

\newtheorem{fact}{Fact}
\newtheorem{lemma}{Lemma}
\newtheorem{proposition}{Proposition}

\newtheorem*{theorem*}{Theorem}

\newtheorem{corollary}{Corollary}

\definecolor{babyblue}{rgb}{0.54, 0.81, 0.94}
\definecolor{citrine}{rgb}{0.89, 0.82, 0.04}
\definecolor{misocolor}{rgb}{0.16,0.27,0.86}

\newcommand{\K}{K}

\newcommand{\samcomplex}[2]{n\pa{#1, #2}}

\newcommand{\R}{\mathbb{R}}
\newcommand{\N}{\mathbb{N}}

\newcommand{\Lb}{L}

\newcommand{\avar}{a}

\newcommand{\Func}{F}
\newcommand{\stateFunc}[1]{\Func_{#1}}

\newcommand{\range}{\R} 
\newcommand{\rewardrange}{[0, 1]}
\newcommand{\gammarange}{[0,1[}

\newcommand{\ouralgo}{{\tt SmoothCruiser}\xspace}
\newcommand{\UCT}{{\tt UCT}\xspace}
\newcommand{\TRPO}{{\tt TRPO}\xspace}
\newcommand{\SSA}{{\tt SSA}\xspace}

\newcommand{\AthreeC}{{\tt A3C}\xspace}

\newcommand{\CommaBin}{\mathbin{\raisebox{0.5ex}{,}}}

\title{Planning in entropy-regularized \\Markov decision processes and games}

\author{%
   Jean-Bastien Grill\thanks{equal contribution}\\
   DeepMind Paris\\
   \texttt{jbgrill@google.com} \\
   \And
   \!\!\!\!\!Omar D.\,Domingues\textcolor{dmred300}{\footnotemark[1]} \\
   \!\!\!\!\!SequeL team, Inria Lille  \\
   \!\!\!\!\!\texttt{omar.darwiche-domingues@inria.fr} \\
   \AND
   \quad Pierre M\'enard \\
   \quad SequeL team, Inria Lille  \\
   \quad \texttt{pierre.menard@inria.fr} \\
   \And
   R\'emi Munos \\
   DeepMind  Paris\\
   \texttt{munos@google.com} 
   \And
   Michal Valko\\
   DeepMind Paris\\
   \texttt{valkom@deepmind.com} \\
}

\begin{document}

\maketitle

\begin{abstract}
We propose \ouralgo, a new planning algorithm for estimating the value function in entropy-regularized Markov decision processes and two-player games, given a generative model of the environment. \ouralgo makes use of the smoothness of the Bellman operator promoted by the regularization to achieve \textit{problem-independent sample complexity} of order $\tilde{\mathcal{O}}(1/\epsilon^4)$ for a desired accuracy~$\epsilon$, whereas for non-regularized settings there are no known algorithms with guaranteed polynomial sample complexity in the  worst case. 
\end{abstract}


\section{Introduction}
Planning with a generative model is \textit{thinking before acting}. An agent thinks  using a world model that it has built from prior experience \citep{sutton1991, Sutton2ndEdition2018}. In the present paper, we study planning in two types of environments, \textit{Markov decision processes} (MDPs) and \textit{two-player turn-based zero-sum games}. In both settings, agents interact with an environment by taking actions and receiving rewards. Each action changes the state of the environment and the agent aims to choose actions to maximize the sum of rewards. We assume that we are given a generative model of the environment, that takes as input a state and an action and returns a reward and a next state as output. Such generative models, called \textit{oracles}, are typically built from known data and involve simulations, for example, a physics simulation. In many cases, simulations are costly. For example, simulations may require the computation of approximate solutions of differential equations or the discretization of continuous state spaces. Therefore, a smart algorithm  \textit{makes only a small the number of oracles calls} required to estimate the value of a state. The total number of oracle calls made by an algorithm is referred to as \textit{sample complexity}. 

The value of a state $s$, denoted by $\ValueFunc(s)$, is the maximum of the sum of discounted rewards that can be obtained from that state. We want an algorithm that returns an estimate of precision $\epsilon$ of the $\ValueFunc(s)$ for any fixed $s$ and  has a low  sample complexity, which should naturally be a function of $\epsilon$. An agent can then use this algorithm to predict the value of the possible actions at any given state and choose the best one. The main advantage in estimating the value of a \textit{single} given state $s$ at a time instead of the complete value function\footnote{as done by approximate dynamic programming} $s \mapsto \ValueFunc(s)$ is that we can have algorithms whose sample complexity does not depend on the size of the state space, which is important when our state space is very large or continuous. On the other hand, the disadvantage is that the algorithm must be run each time a new state is encountered.

Our main contribution is an algorithm that \emph{estimates the value function in a given state} in planning problems that satisfy specific smoothness conditions, which is the case when the rewards are regularized by adding an entropy term. We exploit this smoothness property to obtain a polynomial sample complexity of order $\cOtilde{1/\epsilon^{4}}$ that is \textit{problem independent}. 


\paragraph{Related work} \cite{kearns1999sparse} came up with a sparse sampling algorithm (\SSA) for planning in MDPs with finite actions and arbitrary state spaces. \SSA estimates the value of a state $s$ by building a sparse look-ahead tree starting from $s$. However, \SSA achieves a sample complexity of $\cO{(1/\epsilon)^{\log(1/\epsilon)}}$, which is non-polynomial in $1/\epsilon$. \SSA is slow since its search is \textit{uniform}, i.e., it does not select actions adaptively. \cite{walsh2010integrating} gave an improved version of  \SSA with adaptive action selection, but its sample complexity is still non-polynomial. The \UCT algorithm \citep{kocsis2006bandit}, used for planning in MDPs and games, selects actions based on optimistic estimates of their values and has good empirical performance in several applications. However, the sample complexity of \UCT can be worse than exponential in $1/\epsilon$ for some environments, which is mainly due to exploration issues \citep{coquelin2007bandit}. Algorithms with sample complexities of order $\cO{1/\epsilon^d}$, where $d$ is a problem-dependent quantity, have been proposed for  deterministic dynamics  \citep{hren2008optimistic}, and in an open-loop\footnote{This means that the policy is seen as a function of time, not the states. The open-loop setting is particularly adapted to environments with deterministic transitions.}\,setting \citep{bubeck2010open, leurent2019practical, bartlett2019scale-free},
for bounded number of next states and a full MDP model is known \citep{busoniu2012optimistic}, 
for bounded number of next states in a \textit{finite-horizon} setting \citep{feldman2014simple}, 
for bounded number of next states \citep{szorenyi2014optimistic},  
and for general MDPs \citep{grill2016blazing}. In general, when the state space is infinite and the transitions are stochastic, the problem-dependent quantity $d$ can make the sample complexity guarantees exponential. For a related setting, when rewards are only obtained in the leaves of a fixed tree, \cite{kaufmann2017monte} and \cite{huang2017structured} present algorithms to identify the optimal action in a game based on best-arm identification tools.

Entropy regularization in MDPs and reinforcement learning have been employed in several commonly used algorithms. In the context of policy gradient algorithms, common examples are the \TRPO algorithm  \citep{schulman2015trust} which uses the Kullback-Leibler divergence between the current and the updated policy to constrain the gradient step sizes, the \AthreeC algorithm \citep{mnih2016asynchronous} that penalizes policies with low entropy to improve exploration, and the work of \cite{neu2017unified} presenting a theoretical framework for entropy regularization using the joint state-action distribution. Formulations with entropy-augmented rewards, which is the case in our work, have been used to learn multi-modal policies to improve exploration and robustness \citep{haarnoja2017reinforcement, haarnoja2018soft} and can also be related to policy gradient methods \citep{schulman2017equivalence}.  Furthermore, \cite{geist2019theory} propose a theory of regularized MDPs which includes  entropy as a special case. Summing up, reinforcement learning knows \textit{how} to employ entropy regularization. In this work, we tasked ourselves to give insights on \textit{why}.

\section{Setting and motivation}
\label{s:def}

Both MDPs and two-player games can be formalized as a tuple $(\states, \actions ,\transprob, \rewardfunc, \gamma)$, where $\states$ is the set of states, $\actions$ is the set of actions, $\transprob \triangleq \{\transprob(\cdot|s,a)\}_{s,a \in \states\times\actions}$ is a set of probability distributions over $\states$,  $\rewardfunc: \states\times\actions\to\rewardrange$ is a (possibly random) reward function and $\gamma\in\gammarange$ is the known discount factor. In the MDP case, at each round $t$, an agent is at state $s$, chooses action $a$ and observes a reward $\rewardfunc(s,a)$ and a transition to a next state $\nextstate \sim \transprob(\cdot | s,a)$. In the case of turn-based two-player games, there are two agents and, at each round $t$, an agent chooses an action, observes a reward and a transition;  at round $t+1$ it's the other player's turn. This is equivalent to an MDP with an augmented state space $\states^+ \triangleq \states \times \{1, 2\}$ and transition probabilities such that $\transprob( (\nextstate, j)| (s, i), a) = 0$ if $i = j$. We assume that the action space $\actions$ is finite with cardinality $K$ and the state space $\states$ has arbitrary (possibly infinite) cardinality.  

Our objective is to find an algorithm that outputs a good estimate of the value $\ValueFunc(s)$ for any given state $\svar$ as quickly as possible. An agent can then use this algorithm to choose the best action in an MDP or a game. More precisely, for a state $\svar\in\states$ and given $\epsilon > 0$ and $\pacdelta > 0$, our goal is to compute an estimate $\ValueH{\svar}$ of $\Value{\svar}$ such that
$\PP{\big| \ValueH{\svar} - \Value{\svar}\big| > \epsilon} \le \pacdelta$ with small number of oracle calls required to compute this estimate. In our setting, we consider the case of \textit{entropy-regularized} MDPs and games, where the objective is augmented with an entropy term. 

\subsection{Value functions}

\paragraph{Markov decision process} The policy $\policy$ of an agent is a function from $\states$ to $\cP(\actions)$, the set of probability distributions over $\actions$. We denote by $\policy(a|s)$ the probability of the agent choosing action $a$ at state $s$. In MDPs, the value function at a state $s$, $\ValueFunc(s)$, is defined as the supremum over all possible policies of the expected sum of discounted rewards obtained starting from $s$, which satisfies the Bellman equations \citep{puterman1994markov},
\begin{align}
	\label{eq:value-function-mdp}
	\forall s \in \states, \; \ValueFunc(s) & = \max_{\pi(\cdot|s) \in \cP(\actions)} \EE{ \rewardfunc(s,a) + \gamma \ValueFunc(\nextstate)}, \; a \sim \pi(\cdot|s), \; \nextstate \sim \transprob(\cdot|s,a). 
\end{align}
\paragraph{Two-player turn-based zero-sum games} In this case, there are two agents (1 and~2), each one with its own policy and different goals. If the policy of Agent~2 is fixed, Agent~1 aims to find a policy that \emph{maximizes} the sum of discounted rewards. Conversely, if the policy of Agent~1 is fixed, Agent~2 aims to find a policy that \emph{minimizes} this sum. Optimal strategies for both agents can be shown to exist and for any $(s,i) \in \states^+ \triangleq \states \times \{1, 2\}$, the value function is defined as  \citep{hansen2013strategy}
\begin{align}\label{eq:value-function-games}
	\ValueFunc(s, i) \triangleq%
	\begin{cases}
		\max_{\pi(\cdot|s) \in \cP(\actions)}  \EE{ \rewardfunc((s,i),a) + \gamma \ValueFunc(\nextstate, j)}, &\text{ if } i = 1, \\
		\min_{\pi(\cdot|s) \in \cP(\actions)}  \EE{ \rewardfunc((s,i),a) + \gamma \ValueFunc(\nextstate, j)}, &\text{ if } i = 2,
	\end{cases}
\end{align}
with $a \sim \pi(\cdot|s) $ and $(\nextstate, j) \sim \transprob(\cdot| (s,i), a)$. In this case, the function $s \mapsto \ValueFunc(s, i)$ is the optimal value function for Agent $i$ when the other agent follows its optimal strategy.

\paragraph{Entropy-regularized value functions} Consider a regularization factor $\lambda > 0$. In the case of MDPs, when rewards are augmented by an entropy term, the value function at state $s$ is given by \citep{haarnoja2017reinforcement, dai2018sbeed, geist2019theory}
\begin{align}
\label{eq:reg-value-function-mdp}
\ValueFunc(s) &\triangleq \max_{\pi(\cdot|s) \in \cP(\actions)} \Big \{ \EE{ \rewardfunc(s,a) + \gamma \ValueFunc(\nextstate)} + \lambda \entropy\pa{\pi\pa{\cdot|s}} \Big \}, \; a \sim \pi(\cdot|s), \; \nextstate \sim \transprob(\cdot|s,a) \nonumber \\
&= \lambda \log \sum_{a\in\actions} \exp \pa{ \tfrac{1}{\lambda} \EE{ \rewardfunc(s, a) + \gamma \ValueFunc(\nextstate)}  }, \; \nextstate \sim \transprob(\cdot|s,a),
\end{align}
where $\entropy\pa{\pi\pa{\cdot|s}}$ is the entropy of the probability distribution $\pi(\cdot|s)\in\cP(\actions)$.

The function $\logsumexp{\lambda}: \R^K \to \R$, defined as	$\logsumexp{\lambda}(x)\triangleq \lambda \log\sum_{i=1}^K \exp\pa{x_i/\lambda}$, is a smooth approximation of the $\max$ function, since $\norm{\max - \logsumexp{\lambda}}_\infty \leq \lambda \log K$. Similarly, the function $-\logsumexp{-\lambda}$ is a smooth approximation of the $\min$ function. This allows us to define the regularized version of the value function for turn-based two player games, in which both players have regularized rewards, by replacing the $\max$ and the $\min$ in Equation \ref{eq:value-function-games} by their smooth approximations. 


For any state $s$, let $\stateFunc{s}  \triangleq \logsumexp{\lambda}$ or $\stateFunc{s} \triangleq -\logsumexp{-\lambda}$ depending on $s$. Both for MDPs and games, we can write the entropy-regularized value functions as
\begin{equation}\label{eq:general-value-function}
\ValueFunc(s) = \stateFunc{\svar}\pa{\QFunc_s}, \; \mbox{ with } \; \QFunc_s(a) \triangleq \EE{ \rewardfunc(s, a) + \gamma \ValueFunc(\nextstate)}, \;  \nextstate \sim \transprob(\cdot|s,a),
\end{equation}
where $\QFunc_s \triangleq \pa{\QFunc_s(a)}_{\avar\in\actions}$, the $Q$ function at state $s$, is a vector in $\mathbb{R}^{\K}$ representing the value of each action. The function $\stateFunc{s}$ is the \emph{Bellman operator} at state $s$, which becomes smooth due to the entropy regularization.

\paragraph{Useful properties} Our algorithm exploits the smoothness property of  $\stateFunc{s}$ defined above. In particular, these functions are $L$-smooth, that is, for any $Q, Q' \in \range^K$, we have
\begin{align}
  \abs{\stateFunc{\svar}(Q) - \stateFunc{\svar}(Q') - (Q-Q')\transpose \nabla\stateFunc{\svar}(Q')}
	\le \Lb\norm{Q-Q'}_2^2, \mbox{\ with\ } L =1/\lambda\cdot
\end{align}

Furthermore, the functions $\stateFunc{s}$ have two important properties: $\nabla\stateFunc{\svar}(Q)\footnote{$\nabla\stateFunc{\svar}(Q)$ is the gradient of $\stateFunc{\svar}(Q)$ with respect to $Q$.} \succeq 0$ and  $\normm{\nabla\stateFunc{\svar}(Q)}_1 = 1$ for all $Q \in \R^K$. This implies that the gradient $\nabla\stateFunc{\svar}(Q)$ defines a probability distribution.\footnote{It is a Boltzmann distribution with temperature $\lambda$.}



\paragraph{Assumptions}

We assume that $\states$, $\actions$, $\lambda$, and $\gamma$ are given to the learner. Moreover, we assume that we can access a generative model, the \textit{oracle}, from which we can get reward and transition samples from arbitrary state-action pairs. Formally, when called with parameter $(\svar,\avar)\in\states\times\actions$, the oracle outputs a new random variable $(\rewardrand,\transrand)$ independent from any other outputs received from the generative model so far such that $\transrand \sim \transprob(\cdot|s,a)$ and $\rewardrand$ has same distribution as $\rewardfunc(s,a)$. We denote a call to the oracle as $\rewardrand, \transrand \gets \oracle(s, a).$ 





\subsection{Using regularization for the polynomial sample complexity}
 To pave the road for \ouralgo, we consider two extreme cases, based on the strength of the regularization:
\begin{enumerate}
	\item {\bf Strong regularization\ } In this case, $\lambda \to \infty$ and $L = 0$, that is, $\stateFunc{s}$ is linear for all $s$: $\stateFunc{s}(x) = w_s\transpose x$, with $\normm{w_s}_1 = 1$, $w_s \in \R^k$ and $w_s \succeq 0$,
	\item {\bf No regularization\ } In this case, $\lambda=0$ and $L \to \infty$, that is, $\stateFunc{s}$ cannot be well approximated by a linear function.\footnote{This is the case of the $\max$ and $\min$ functions.}
\end{enumerate}
In the strongly regularized case, we can approximate the value $\ValueFunc(s)$ with $\tilde{\mathcal{O}}(1/\epsilon^2)$ oracle calls. This is due to the linearity of $\stateFunc{s}$, since the value function can be written as $\ValueFunc(s) =   \EE{\sum_{t=0}^\infty \gamma^t \rewardfunc(S_t, A_t)\mid S_0 = s}$ where $A_t$ is distributed according to the probability vector $w_{S_t}$. As a result, $\ValueFunc(s)$ can be estimated by Monte-Carlo sampling of trajectories.


With no regularization, we can apply a simple adaptation of the sparse sampling algorithm of \cite{kearns1999sparse} that we briefly describe. Assume that we have an subroutine that provides an approximation of the value function with precision $\epsilon/\sqrt{\gamma}$, denoted by $\ValueFuncH_{\epsilon/\sqrt{\gamma}}(s)$, for any $s$. We can call this subroutine several times as well as the oracle to get improved estimate $\ValueFuncH$ defined as
\begin{align*}
\ValueFuncH(s) = \stateFunc{s}\pa{\QFuncHat_s}\quad \mbox{ with }\quad \QFuncHat_s(a) \gets  \frac{1}{N} \sum_{i=1}^N \left[ r_i(s,a)  + \gamma \ValueFuncH_{\epsilon/\sqrt{\gamma}}(\nextstate_i) \right],
\end{align*}
where $r_i(s,a)$ and $\nextstate_i$ are rewards and next states sampled by calling the oracle  with parameters $(s,a)$. By Hoeffding's inequality, we can choose $N = \cO{1/\epsilon^2}$ such that $\ValueFuncH(s)$ is an approximation of $\ValueFunc(s)$ with precision $\epsilon$ with high probability. By applying this idea recursively, we start with $\ValueFuncH = 0$, which is an approximation of the value function with precision $1/(1-\gamma)$, and progressively improve the estimates towards a desired precision $\epsilon$, which can be reached at a recursion depth of $H = \cO{\log(1/\epsilon)}$. Following the same reasoning as \cite{kearns1999sparse},  this approach has a sample complexity of $\cO{(1/\epsilon)^{\log(1/\epsilon)}}$: to estimate the value at a given recursion depth, we make $\cO{1/\epsilon^2}$ recursive calls  and stop once we reach the maximum depth, resulting in a sample complexity of
\begin{align*}
    \underbrace{\frac{1}{\epsilon^2} \times \cdots \times \frac{1}{\epsilon^2}}_{\cO{\log(1/\epsilon)} \mbox{ times } } = \pa{\frac{1}{\epsilon}}^{\cO{ \log\pa{\frac{1}{\epsilon}} }}\cdot
\end{align*} 
In the next section, we provide $\ouralgo$ (Algorithm \ref{alg:ouralgo}), that uses the assumption that the functions $\stateFunc{s}$ are $L$-smooth with $0 < L < \infty$ to interpolate between the two cases above and obtain a sample complexity of $\cOtilde{1/\epsilon^{4}}$.

\section{SmoothCruiser}
We now describe our planning algorithm. Its building blocks are two procedures, $\estV$ (Algorithm~\ref{alg:samplev}) and $\estQ$ (Algorithm \ref{alg:estq}) that recursively call each other. The procedure $\estV$ returns a noisy estimate of $\ValueFunc(s)$ with a bias bounded by $\epsilon$. The procedure $\estQ$  averages the outputs of several calls to $\estV$ to obtain an estimate $\QFuncHat_s$ that is an approximation of $\QFunc_s$ with precision $\epsilon$ with high probability. Finally, $\ouralgo$ calls $\estQ(s, \epsilon)$ to obtain $\QFuncHat_s$ and outputs $\ValueFuncH(s) = \stateFunc{s}(\QFuncHat_s)$. Using the assumption that $\stateFunc{s}$ is 1-Lipschitz, we can show that $\ValueFuncH(s)$ is an approximation of $\ValueFunc(s)$ with precision $\epsilon$. Figure \ref{fig:smoothcruiser} illustrates a call to $\ouralgo$.

\begin{figure}[ht!]
    \centering
    \includegraphics[width=0.6\textwidth]{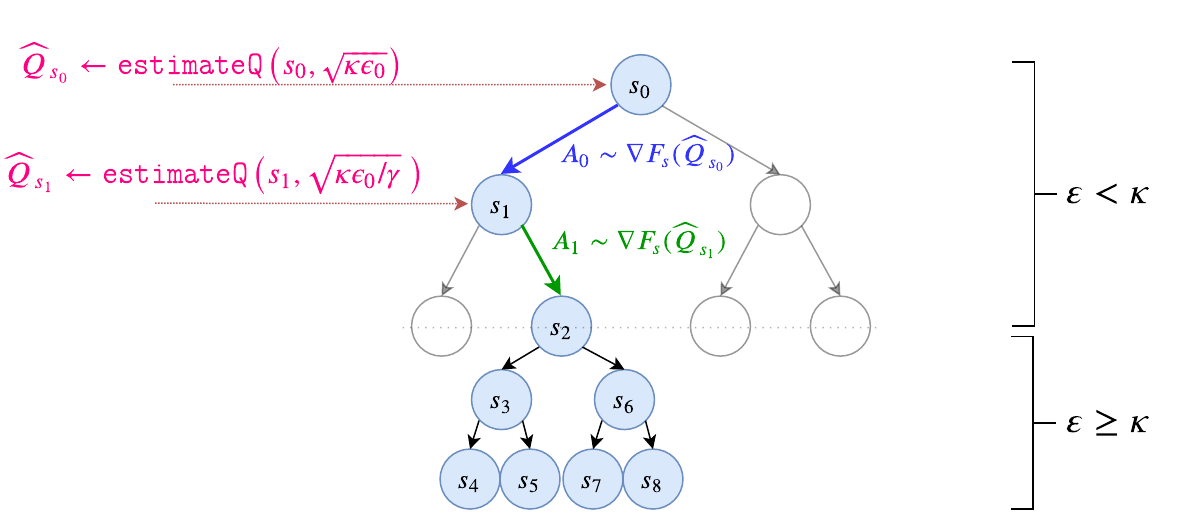}
    \caption{Visualization of a call to $\ouralgo(s_0, \varepsilon_0,\algdelta)$.}
    \label{fig:smoothcruiser}
\end{figure}

\subsection{Smooth sailing}
\begin{wrapfigure}{r}{0.46\textwidth}
  \vspace{-8mm}
  \begin{minipage}{0.46\textwidth}

	\begin{algorithm}[H] 
		\centering
		\caption{\ouralgo}\label{alg:ouralgo}
		\footnotesize
		\begin{algorithmic}
			\State {\bfseries Input:} $(\svar, \epsilon,\algdelta)\in\states\times\R_+\!\times\R_+$
			\State $\supFsZero \gets \sup_{s\in\states}|\stateFunc{\svar}(0)| = \lambda\log K $ 
			\State $\barepsilon \gets (1-\sqrt{\gamma})/(KL)$
			\State Set $\algdelta$, $\barepsilon,$ and $\supFsZero$ as  global parameters
			\State $\QFuncHat_s  \gets \estQ(\svar,\epsilon)$
			\State {\bfseries Output:} $\stateFunc{\svar}\pa{\QFuncHat_s }$
		\end{algorithmic}
	\end{algorithm}
\end{minipage}
\vspace{-4mm}
\end{wrapfigure}




The most important part of the algorithm is the procedure $\estV$, that returns a low-bias estimate of the value function. Having the estimate of the value function, the procedure $\estQ$ averages the outputs of $\estV$ to obtain a good estimate of the $Q$ function with high probability. The main idea of $\estV$ is to first compute an estimate of precision $\mathcal{O}(\sqrt{\epsilon})$ of the value of each action $\{\QFuncHat_s(a)\}_{a \in \actions}$ to linearly approximate the function $\stateFunc{\svar}$ around $\QFuncHat_s.$ The  local approximation of $\stateFunc{\svar}$ around $\QFuncHat_s$ is subsequently used to estimate the value of $s$ with a better precision, of order $\mathcal{O}(\epsilon)$, which is possible due to the smoothness of $\stateFunc{\svar}$. 
\vskip -0.5em
\begin{minipage}[t]{0.45\textwidth}
	\begin{algorithm}[H] 
		\centering
		\caption{\estV}\label{alg:samplev}
		\footnotesize
		\begin{algorithmic}[1]
			\State {\bfseries Input:} $(\svar, \epsilon)\in\states\times\R_+$
			\If {$\epsilon \ge (1+M_\lambda)/(1-\gamma)$}
			\State {\bfseries Output:} $0$
			\ElsIf {$\epsilon \geq \barepsilon$}
			\State $ \QFuncHat_s \gets \estQ(\svar,\epsilon)$
			\State {\bfseries Output:} $\stateFunc{\svar}\pa{\QFuncHat_s}$
			\ElsIf {$\epsilon < \barepsilon$}
			\State $\QFuncHat_s  \gets \estQ(\svar,\sqrt{\barepsilon\epsilon})$
			\State $A \gets $ action drawn from $\nabla\stateFunc{\svar}\pa{\QFuncHat_s}$ 
			\State $(\rewardrand, \transrand) \gets \oracle(\svar, A)$
			\State $\hat{V} \gets \estV(\transrand,\epsilon/\sqrt{\gamma})$
			\State \textbf{Output:} 
			\State \quad {$\stateFunc{\svar}\pa{\QFuncHat_s } - \QFuncHat_s \transpose\nabla\stateFunc{\svar}\pa{\QFuncHat_s } + (\rewardrand + \gamma\hat{V})$ } 
			\EndIf
		\end{algorithmic}
	\end{algorithm} 
\end{minipage}\hfill
\begin{minipage}[t]{0.46\textwidth}
	\begin{algorithm}[H]  
		\centering
		\caption{\estQ}\label{alg:estq}
		\footnotesize
		\begin{algorithmic}[1]
			\State {\bfseries Input:} $(\svar, \epsilon)$
			\State $N(\epsilon) \gets \ceil{\frac{18(1+\supFsZero)^2}{(1-\gamma)^4(1-\sqrt{\gamma})^2}  \frac{\log\pa{2K/\algdelta}}{\epsilon^2}}$
			\For{$\avar\in\actions$}
			\State $q_i \gets 0$ for $i\in{1,...,N(\epsilon)}$
			\For{$i\in 1,...,N(\epsilon)$}
			\State $(\rewardrand, \transrand) \gets$ $\oracle(\svar, \avar)$. 
			\State $\hat{V} \gets \estV\pa{\transrand,\epsilon/\sqrt{\gamma}}$
			\State $q_i \gets \rewardrand + \gamma\hat{V}$
			\EndFor
			\State $\QFuncHat_s(a) \gets \mathbf{mean}\pa{q_1, \ldots, q_N}$
			\State // clip $\QFuncHat_s(a)$ to $[0,(1+M_\lambda)/(1-\gamma)]$
			\State $\QFuncHat_s(a) \gets \max(0, \QFuncHat_s(a))$
			\State $\QFuncHat_s(a) \gets \min((1+\supFsZero)/(1-\gamma), \QFuncHat_s(a))$
			\EndFor
			\State \textbf{Output:} $\QFuncHat_s$ 
		\end{algorithmic}
	\end{algorithm} 
\end{minipage}

\vskip 1em

For a target accuracy $\epsilon$ at state $s$, $\estV$ distinguishes three cases, based on a reference threshold $\barepsilon \triangleq (1-\sqrt{\gamma})/(KL)$, which is  the maximum value of $\epsilon$ for which we can compute a good estimate of the value function using linear approximations of $\stateFunc{\svar}$.

\begin{itemize}
    \item 

\textbf{First,} if $\epsilon \geq (1+\lambda\log K)/(1-\gamma)$, then 0 is a valid output, since $\ValueFunc(s)$ is bounded by $(1+\lambda\log K)/(1-\gamma)$. This case furthermore ensures that our algorithm terminates, since the recursive calls are made with increasing values of $\epsilon$.
    \item 
\textbf{Second,} if $\barepsilon \leq \epsilon \leq (1+\lambda\log K)/(1-\gamma)$, we run  $ \stateFunc{s}\pa{\estQ(s, \epsilon)}$ in which for each action, both the oracle and $\estV$ are called $\cO{1/\epsilon^2}$ times in order to return $\hat V(s)$ which is with high probability  an $\epsilon$-approximation of $\ValueFunc(s).$ 
    \item 
\textbf{Finally,} if $\epsilon < \barepsilon$, we take advantage of the smoothness of $\stateFunc{s}$ to compute an $\epsilon$-approximation of $\ValueFunc(s)$ in a more efficient way than calling the oracle and $\estV$ $\cO{1/\epsilon^2}$ times. We achieve it by calling $\estQ$ with a precision $\sqrt{\barepsilon\epsilon}$ instead of $\epsilon$, which requires $\cO{1/\epsilon}$ calls instead. 
\end{itemize}

\subsection{Smoothness guarantee an improved sample complexity}
In this part, we describe the key ideas that allows us to exploit the smoothness of the Bellman operator to obtain a better sample complexity. Notice that when $\epsilon < \barepsilon$, the procedure $\estQ$ is called to obtain an estimate $\QFuncHat_s$ such that \[\normm{\QFuncHat_s-\QFunc_s}_2 =  \cO{\sqrt{\epsilon/L}}.\] The procedure $\estV$ then continues with computing a linear approximation of $\stateFunc{s}(\QFunc_s)$ around~$\QFuncHat_s$. Using  the $L$-smoothness of $\stateFunc{s},$
we guarantee the $\eps$-approximation,

\begin{align*}
\abs{ \stateFunc{s}(\QFunc_s) - \left\lbrace \stateFunc{s}(\QFuncHat_s) + (\QFunc_s-\QFuncHat_s)\transpose\nabla\stateFunc{s}(\QFuncHat_s)\right\rbrace  } \leq L \norm{\QFuncHat_s-\QFunc_s}_2^2  = \cO{\epsilon}.
\end{align*}

We wish to output this linear approximation, but we need to handle the fact that the vector $\QFunc_s$ (the true $Q$-function at $s$) is unknown. Notice that the vector $\nabla\stateFunc{s}(\QFuncHat_s)$  represents a probability distribution. The term $\QFunc_s\transpose \nabla\stateFunc{s}(\QFuncHat_s)$  in the linear approximation of $\stateFunc{s}(\QFunc_s)$ above can be expressed as 
\begin{align*}
\QFunc_s\transpose \nabla\stateFunc{s}(\QFuncHat_s) = \EE{ \QFunc_s(A) \given \QFuncHat_s}, \; \mbox{ with } A \sim \nabla\stateFunc{s}(\QFuncHat_s).
\end{align*}
Therefore, we can build a low-bias estimate of $\QFunc_s\transpose \nabla\stateFunc{s}(\QFuncHat_s)$ from estimating only $\QFunc_s(A)$:
\begin{itemize}
	\item sample action $A \sim \nabla\stateFunc{s}(\QFuncHat_s)$
	\item call the generative model to sample a reward and a next state $\rewardrand_{s, A}, \transrand_{s, A} \gets \oracle(s, A)$
	\item obtain an $\cO{\epsilon}$-approximation of $\QFunc_s(A)$: $\tilde{Q}(A) = \rewardrand_{s, A} + \gamma \estV\pa{\transrand_{s, A},\epsilon/\sqrt{\gamma}}$
	\item output $\ValueFuncH(s) = \stateFunc{\svar}(\QFuncHat_s) - \QFuncHat_s \transpose\nabla\stateFunc{\svar}(\QFuncHat_s ) + \tilde{Q}(A)$
\end{itemize}

We show that  $\ValueFuncH(s)$ is an $\epsilon$-approximation of the true value function $\ValueFunc(s)$. The benefit of such approach is that we can call $\estQ$ with a precision $\cO{\sqrt{\epsilon}}$ instead of $\cO{\epsilon}$, which thanks to the smoothness of $\stateFunc{s}$,  reduces the sample complexity. In particular, one call to $\estV(s, \epsilon)$ will make $\cO{1/\epsilon}$ recursive calls to $\estV(s, \cO{\sqrt{\epsilon}})$, and the total number of calls to $\estV$ behaves as
\begin{align*}
    \frac{1}{\epsilon} \times \frac{1}{\epsilon^{1/2}}\times \frac{1}{\epsilon^{1/4}} \times \cdots \leq \frac{1}{\epsilon^2}\cdot
\end{align*}

Therefore, the number of $\estV$ calls made by \ouralgo is of order $\cO{1/\epsilon^2}$, which implies that the total sample complexity is of  $\mathcal{O}(1/\epsilon^4).$

\subsection{Comparison to Monte-Carlo tree search}
\label{sec:comparison-to-mcts}

 \begin{wrapfigure}{r}{0.46\textwidth}
  \vspace{-8mm}
\begin{minipage}[t]{0.45\textwidth}
	\begin{algorithm}[H] 
		\centering
		\caption{\genericmcts}\label{alg:generic-mcts}
		\footnotesize
		\begin{algorithmic}
			\State {\bfseries Input:} state $s$
			\Repeat \;  $\mctssearch(s, 0)$
			\Until timeout
			\State {\bf Output:} estimate of best action or value.
		\end{algorithmic}
	\end{algorithm}
\end{minipage}\hfill
\end{wrapfigure}



Several planning algorithms are based on Monte-Carlo tree search (MCTS, \citealp{coulom2007efficient, kocsis2006bandit}).
 Algorithm~\ref{alg:generic-mcts} gives a template for MCTS, which uses the procedure $\mctssearch$ that calls $\selectaction$ and $\evalleaf$. Algorithm~\ref{alg:mcts-search},  $\mctssearch$, returns an estimate of the value function; $\selectaction$ selects the action to be executed (also called \emph{tree policy}); and $\evalleaf$ returns an estimate of the value of a leaf. We now provide the analogies that make it possible to see $\ouralgo$ as an MCTS algorithm:
 

\begin{itemize}
	\item $\estV$ corresponds to the function $\mctssearch$
	\item $\selectaction$ is implemented by calling $\estQ$ to compute $\QFuncHat_s$, followed by sampling an action with probability proportional to $\nabla\stateFunc{s}(\QFuncHat_s)$
	\item $\evalleaf$ is implemented using the sparse sampling strategy of \cite{kearns1999sparse}, if we see leaves as the nodes reached when $\epsilon \geq \barepsilon$
\end{itemize}

\section{Theoretical guarantees}
\label{s:gua}

\begin{wrapfigure}{r}{0.46\textwidth}
\vspace{-8mm}
\begin{minipage}[t]{0.45\textwidth}
\vspace{-4mm}
	\begin{algorithm}[H] 
		\centering
		\caption{\mctssearch}\label{alg:mcts-search}
		\footnotesize
		\begin{algorithmic}
			\State {\bfseries Input:} state $s$, depth $d$
			\If{$d > d_\mathrm{max}$}
			\State {\bf Output:} $\evalleaf(s)$
			\EndIf 
			\State $a \gets \selectaction(s, d)$
			\State $R, Z \gets \oracle(s, a)$
			\State {\bf Output:} $R + \gamma\mctssearch(Z, d+1)$
		\end{algorithmic}
	\end{algorithm}
\end{minipage}
\vspace{-4mm}
\end{wrapfigure}
In Theorem~\ref{thm:sample-complexity} we bound the sample complexity. Note that \ouralgo is non-adaptive, hence its sample complexity is \emph{deterministic} and \emph{problem independent}. Indeed, since our algorithm is agnostic to the output of the oracle, it performs the same number of oracle calls for any given~$\epsilon$ and $\algdelta$, regardless of the \emph{random} outcome of these calls.
\begin{restatable}{theorem}{thmsample}\label{thm:sample-complexity}
	Let $\samcomplex{\epsilon}{\algdelta}$ be the number of oracle calls before \ouralgo terminates. For any state $\svar\in\states$ and $\epsilon, \algdelta > 0$,
	\begin{align*}
		\samcomplex{\epsilon}{\algdelta} & \leq \frac{c_1}{\epsilon^4}\log\pa{\frac{c_2}{\algdelta}} \left[ c_3\log\pa{\frac{c_4}{\epsilon}} \right]^{ \log_2\pa{c_5 \pa{ \log\pa{\frac{c_2}{\algdelta}}}}} 
		 = \cOtilde{ \frac{1}{\epsilon^{4}}}\CommaBin 
	\end{align*}
	where $c_1, c_2, c_3, c_4$, and $c_5$ are constants that depend only on $K$, $L$, and $\gamma$. 
\end{restatable}
The proof of Theorem~\ref{thm:sample-complexity} with the exact constants is in the appendix.  In Theorem~\ref{thm:consistency}, we  provide our consistency result, stating that the output of \ouralgo applied to a state $\svar\in\states$ is a good approximation of $\Value{\svar}$ with high probability.

\begin{restatable}{theorem}{thmpac}\label{thm:consistency}
	For any $\svar\in\states$, $\epsilon > 0,$ and $\pacdelta > 0$, there exists a $\algdelta$ that depends on $\epsilon$ and $\pacdelta$ such that the output $\ValueH{\svar}$ of $\ouralgo(s, \epsilon, \algdelta)$ satisfies
	\[\PP{\big\lvert\ValueH{\svar} - \Value{\svar}\big\rvert > \epsilon} \le \pacdelta.\]
	and such that $\samcomplex{\epsilon}{\algdelta} = \cO{ 1/\epsilon^{4+c}}$ for any $c > 0$.
\end{restatable}
\noindent

More precisely, in the proof of Theorem~\ref{thm:consistency}, we establish that \[\PP{\big\lvert\ValueH{\svar} - \Value{\svar}\big\rvert > \epsilon} \le \algdelta \samcomplex{\epsilon}{\algdelta}.\] Therefore, for any parameter $\algdelta$ satisfying $\algdelta \samcomplex{\epsilon}{\algdelta} \leq \pacdelta$, \ouralgo{} with parameters $\epsilon$ and $\algdelta$ provides an approximation of $V(s)$ which is $(\epsilon,\pacdelta)$ correct.

\paragraph{Impact of regularization constant} For a regularization constant $\lambda$, the smoothness constant is $L = 1/\lambda$. in Theorem \ref{thm:sample-complexity} we did not make the dependence on $L$ explicit to preserve simplicity. However, it easy to analyze the sample complexity in the two limits:
\begin{description}
    \item[\hspace{3cm} strong regularization] $L \to 0$ and  $\stateFunc{s}$ is linear 
    \item[\hspace{3cm} no regularization] $L \to \infty$ and $\stateFunc{s}$ is not smooth
\end{description}

As $L \to 0$, the condition $\barepsilon \leq \epsilon \leq (1+\lambda\log K)/(1-\gamma)$ will be met less and eventually the algorithm will sample $N = \cO{1/\epsilon^2}$ trajectories, which implies a sample complexity of order $\cO{1/\epsilon^2}.$ On the other hand, as $L$ goes to $\infty$, the condition $\epsilon < \barepsilon$ will be met less and the algorithm eventually runs a uniform sampling strategy of \citet{kearns1999sparse}, which results in a sample complexity of order $\cO{(1/\epsilon)^{\log(1/\epsilon)}}$, which is non-polynomial in $1/\epsilon$. 

Let $\ValueFunc_\lambda(s)$ be the entropy regularized value function and $\ValueFunc_0(s)$ be its non-regularized version. Since $\stateFunc{s}$ is 1-Lipschitz and $\normm{\logsumexp{\lambda} - \max}_\infty \leq \lambda\log K$, we can prove that $\sup_s|\ValueFunc_\lambda(s) - \ValueFunc_0(s) | \leq \lambda \log K/(1-\gamma)$.  Thus, we can interpret $\ValueFunc_\lambda(s)$ as an approximate value function which we can estimate faster.

\paragraph{Comparison to lower bound} For non-regularized problems, \cite{kearns1999sparse} prove a sample complexity lower bound of $\Omega\pa{(1/\epsilon)^{1/\log(1/\gamma)}}$, which is polynomial in $1/\epsilon$, but its exponent grows as $\gamma$ approaches $1$. For regularized problems, Theorem \ref{thm:sample-complexity} shows that the sample complexity is polynomial with an exponent that is \emph{independent} of $\gamma$. Hence, when $\gamma$ is close to 1, regularization gives us a better asymptotic behavior with respect to $1/\epsilon$ than the lower bound for the non-regularized case, although we are not estimating the same value.

\section{Generalization of SmoothCruiser}
\label{generalizations}

Consider the general definition of value functions in Equation \ref{eq:general-value-function}. Although we focused on the case where $\stateFunc{s}$ is the $\logsumexp{}$ function, which arises as a consequence of entropy regularization, our theoretical results hold for any set of functions $\{\stateFunc{s}\}_{s \in \states}$ that  for any $s$ satisfy the following conditions:

\begin{enumerate}
    \item $\stateFunc{s}$ is differentiable
	\item  $\forall Q \in \range^K, 0 < \normm{\nabla\stateFunc{\svar}(Q)}_1 \leq 1$ 
	\item (nonnegative gradient) $\forall Q \in \range^K, \nabla\stateFunc{\svar}(Q) \succeq 0$ 
	\item ($\Lb$-smooth) there exists $\Lb\ge 0$ such that for any $Q, Q' \in \range^K$
	\[\abs{\stateFunc{\svar}(Q) - \stateFunc{\svar}(Q') - (Q-Q')\transpose \nabla\stateFunc{\svar}(Q')}
	\le \Lb\norm{Q-Q'}_2^2\]
\end{enumerate}

For the more general definition above, we need to make two simple modifications of the procedure $\estV.$  When $\epsilon < \barepsilon$, the action $A$ in $\estV$ is sampled according to \[A \sim \frac{\nabla\stateFunc{\svar}(\QFuncHat_s)}{\normm{\nabla\stateFunc{\svar}(\QFuncHat_s)}_1}\] 
and its output is modified to 
\[\stateFunc{\svar}(\QFuncHat_s) - \QFuncHat_s\transpose\nabla\stateFunc{\svar}(\QFuncHat_s) + (\rewardrand + \gamma\hat{v}) \normm{\nabla\stateFunc{\svar}(\QFuncHat)}_1. \]

In particular, \ouralgo can be used for more general regularization schemes, as long as the Bellman operators satisfy the assumptions above. One such example is presented in Appendix~\ref{sec:other-smooth-approx}.

\section{Conclusion}

We provided \ouralgo, an algorithm that estimates the value function of MDPs and discounted games defined through smooth approximations of the optimal Bellman operator, which is the case in entropy-regularized value functions. More generally, our algorithm can also be used when value functions are defined through \textit{any} smooth Bellman operator with nonnegative gradients. We showed that our algorithm has a polynomial sample complexity of $\tilde{\mathcal{O}}(1/\epsilon^{4})$, where $\epsilon$ is the desired precision. This guarantee is problem independent and holds for \textit{state spaces of arbitrary cardinality.} 

One interesting interpretation of our results is that computing entropy-regularized value functions, which are commonly employed for reinforcement learning, can be seen as a smooth relaxation of a planning problem for which we can obtain a much better sample complexity in terms of the required precision $\epsilon$. Unsurprisingly, when the regularization tends to zero, we recover the well-known non-polynomial bound $\cO{(1/\epsilon)^{\log(1/\epsilon)}}$ of \cite{kearns1999sparse}. Hence, an interesting direction for future work is to study adaptive regularization schemes in order to accelerate planning algorithms. Although \ouralgo makes large amount of recursive calls, which makes it impractical in most situations, we believe it might help us to understand how regularization speeds planning and inspire more practical algorithms. This might be possible by exploiting its similarities to Monte-Carlo tree search that we have outlined above.

\paragraph{Acknowledgments}
The research presented was supported by European CHIST-ERA project DELTA, French Ministry of
Higher Education and Research, Nord-Pas-de-Calais Regional Council,
Inria and Otto-von-Guericke-Universit\"at Magdeburg associated-team north-European project Allocate, and French National Research Agency project BoB (grant n.ANR-16-CE23-0003), 
FMJH Program PGMO with the support of this program from Criteo.

\newpage
\bibliography{library,b,ref}
\bibliographystyle{plainnat}

\newpage

\appendix

\section{Preliminaries}

\subsection{General definition of value functions}

We consider the general definition of value functions in Equation \ref{eq:general-value-function} and we assume that all the functions $\stateFunc{s}$ satisfy

\begin{enumerate}
    \item $\stateFunc{s}$ is differentiable,
   \item  $\forall x \in \range^K, 0 < \normm{\nabla\stateFunc{\svar}(x)}_1 \leq 1,$ 
   \item (nonnegative gradient) $\forall x \in \range^K, \nabla\stateFunc{\svar}(x) \succeq 0,$ 
   \item ($\Lb$-smooth) There exists $\Lb\ge 0$ such that for any $x_0, x \in \range^K,$
   \[\abs{\stateFunc{\svar}(x) - \stateFunc{\svar}(x_0) - (x-x_0)\transpose \nabla\stateFunc{\svar}(x_0)}
   \le \Lb\norm{x-x_0}_2^2,\]
\end{enumerate}

which is the case for the functions $\logsumexp{\lambda}$ and $-\logsumexp{-\lambda}$ that we study in the present paper. In particular, the second requirement implies that $\stateFunc{s}$ is 1-Lipschitz,
\[
    \forall x, y \in \range^K, \abs{F_s(x)-F_s(y)}\leq\norm{x-y}_\infty.
\]
For this more general definition, we modify the output of $\estV$ when $\epsilon < \barepsilon$ to
\[
     \mathbf{output} = \stateFunc{\svar}\pa{\QFuncHat_s} - (\QFuncHat_s)\transpose\nabla\stateFunc{\svar}\pa{\QFuncHat_s } + (\rewardrand + \gamma\hat{v}) \norm{\nabla\stateFunc{\svar}\pa{\QFuncHat }}_1
\]
and the action sampled in $\estV$ is sampled according to
\[
    A \sim \frac{\nabla\stateFunc{\svar}\pa{\QFuncHat_s}}{\norm{\nabla\stateFunc{\svar}\pa{\QFuncHat_s}}_1}
\quad \text{instead of} \quad A \sim \nabla\stateFunc{\svar}\pa{\QFuncHat_s}.
\]

\subsection{Other definitions}
The constant $\supFsZero$ is defined as
\[
    \supFsZero \triangleq \sup_{s\in\states} |\stateFunc{\svar}(0)|.
\]
For any $c \in \R$, the function $\mathrm{clip}_c: \R^d \to \R^d$ is defined component-wise as 
\begin{align*}
    \clip{c}{x}_i = %
    \begin{cases}
        0 \; & \mbox{ if } x_i \leq 0, \\
        x_i \; & \mbox{ if } - c < x_i < c,\\
        c\; &  \mbox{ if } x \geq c.
    \end{cases}
\end{align*}
\section{Sample complexity}
\thmsample*
To bound the sample complexity, we make the following steps.
\begin{itemize}
   \item Proposition~\ref{prop:sample-complexity-large-epsilon} bounds the number of recursive calls of $\estV$ in the uniform sampling phase ($\epsilon \geq \barepsilon$) and is similar to the results of \cite{kearns1999sparse}.
   \item Lemma~\ref{lemma:sample-complexity} bounds the number of recursive calls of $\estV$ when $\epsilon <\barepsilon.$
   \item By noticing that the number of recursive calls of $\estV$ is equal to the number of oracle calls, we bound the sample complexity of $\ouralgo$ in Theorem~\ref{thm:sample-complexity}.
\end{itemize}

Let $\nsamplev(s, \epsilon,\algdelta)$ be the total number of recursive calls to $\estV$ after an initial call with parameters $(s, \epsilon)$, and including the initial call. Since this number does not depend on the state $s$, we denote it by $\nsamplev(\epsilon,\algdelta)$. 

\begin{proposition}
   \label{prop:sample-complexity-large-epsilon}
   Let $\epsilon \geq \barepsilon$. For all $h \in \N$, $\forall \epsilon $ such that $\frac{(1+M_\lambda)\sqrt{\gamma}^h}{1-\gamma} \leq \epsilon \leq \frac{1+M_\lambda}{1-\gamma}$, we have 
   \begin{align*}
      \nsamplev(\epsilon,\algdelta) & \leq \gamma^{\frac{1}{2}H(\epsilon)(H(\epsilon)-1)}\pa{\frac{2\alpha(\algdelta)}{\epsilon^2}}^{H(\epsilon)} \\
      & \leq \gamma^{\frac{1}{2}H(\barepsilon)(H(\barepsilon)-1)}\pa{\frac{2\alpha(\algdelta)}{\barepsilon^2}}^{H(\barepsilon)} 
   \end{align*}
   where 
   \begin{align*}
      H(\epsilon) = \left\lceil2\log_\gamma \pa{\frac{\epsilon(1-\gamma)}{1+M_\lambda}} \right\rceil
   \end{align*}
   and 
   \begin{align*}
      \alpha(\algdelta) = \frac{18(1+\supFsZero)^2 K}{(1-\gamma)^4 (1-\sqrt{\gamma})^2} \log\pa{ \frac{2K}{\algdelta} }
   \end{align*}
   
\end{proposition}
\begin{proof}
   We want to prove that $\nsamplev(\epsilon,\algdelta) \leq G(\epsilon)$, where
   
   \begin{align*}
      G(\epsilon) = \gamma^{\frac{1}{2}H(\epsilon)(H(\epsilon)-1)}\pa{\frac{2\alpha(\algdelta)}{\epsilon^2}}^{H(\epsilon)} 
   \end{align*}
   
   We proceed by induction on $h$.
   
   \paragraph{Base case} Let $h = 0$. We have $\epsilon = \frac{1+M_\lambda}{1-\gamma}$, which implies $\nsamplev(\epsilon,\algdelta) = 1$ and $G(\epsilon) = 1$ (since $H(\epsilon)=0$). Hence, the proposition is true for $h=0$.
   \paragraph{Induction hypothesis} Assume true for $h$.
   \paragraph{Induction step} Let $\epsilon \geq \frac{(1+M_\lambda)\sqrt{\gamma}^{h+1}}{1-\gamma}\cdot$ Since $\frac{\epsilon}{\sqrt{\gamma}} \geq \frac{(1+M_\lambda)\sqrt{\gamma}^h}{1-\gamma}\CommaBin$ we use the induction hypothesis to obtain
   \begin{align*}
       \nsamplev(\epsilon,\algdelta) & = \underbrace{1}_{\mbox{current call}} + \underbrace{K N(\epsilon) \nsamplev\pa{\frac{\epsilon}{\sqrt{\gamma}}, \algdelta}}_{\mbox{calls in $\estQ$}} \\
      & \leq \frac{2\alpha(\algdelta)}{\epsilon^2} \nsamplev\pa{\frac{\epsilon}{\sqrt{\gamma}}, \algdelta} \\
      & \leq \frac{2\alpha(\algdelta)}{\epsilon^2} \gamma^{\frac{1}{2}(H(\epsilon)-1)(H(\epsilon)-2)}\pa{\frac{\gamma 2\alpha(\algdelta)}{\epsilon^2}}^{H(\epsilon)-1}, \quad \mbox{ since } H\pa{\frac{\epsilon}{\sqrt{\gamma}}} = H(\epsilon) - 1 \\
      & = \gamma^{\frac{1}{2}H(\epsilon)(H(\epsilon)-1)}\pa{\frac{2\alpha(\algdelta)}{\epsilon^2}}^{H(\epsilon)}\CommaBin
   \end{align*}
   which completes the proof. 
\end{proof}

\begin{lemma}
   \label{lemma:sample-complexity}
   Let $\epsilon \leq \barepsilon$. For all $h \in \N$, $\forall \epsilon \geq \barepsilon \sqrt{\gamma}^h$, we have 
   \begin{align*}
   \nsamplev(\epsilon,\algdelta) \leq \eta_1 \left[ \log_{\frac{1}{\gamma}}\pa{\frac{\barepsilon/\gamma}{\epsilon}} \right]^{\eta_2(\algdelta)} \frac{1}{\epsilon^2}
   \end{align*}
   where 
   \begin{align*}
      & \barepsilon = \frac{1-\sqrt{\gamma}}{KL} \\ 
      & \eta_1 = \barepsilon^2 \nsamplev(\barepsilon,\algdelta) \\
      & \eta_2(\algdelta) = \log_2\pa{ \frac{\gamma}{1-\gamma}  \frac{2\beta(\algdelta)}{\barepsilon} } \\
      & \beta(\algdelta) = \frac{18(1+\supFsZero)^2 K^2 L}{(1-\gamma)^4 (1-\sqrt{\gamma})^3} \log\pa{ \frac{2K}{\algdelta} } 
   \end{align*}
   under the condition that
   \begin{align}
      \log_2\pa{ \frac{\gamma}{1-\gamma} \frac{2\beta(\algdelta)}{\barepsilon} } \geq 0, \quad \mbox{i.e.}, \quad \beta(\algdelta) \geq \frac{(1-\gamma)(1-\sqrt{\gamma})}{2\gamma KL}
   \end{align}
   which is satisfied by choosing $\algdelta$ small enough. 
\end{lemma}
\begin{proof}
   First, let us define some auxiliary quantities,
   \begin{align}
      & B_1(\epsilon) \triangleq \left[ \log_{\frac{1}{\gamma}}\pa{\frac{\barepsilon/\gamma}{\epsilon}} \right]^{ \eta_2(\algdelta) }, \\
      & B_2(\epsilon) \triangleq \frac{\eta_1}{\epsilon^2} \quad \mbox{ and } \\
      & B(\epsilon) \triangleq B_1(\epsilon) B_2(\epsilon)
   \end{align}
   
   We want to prove that $\nsamplev(\epsilon,\algdelta) \leq B(\epsilon)$ and we proceed by induction on $h$.
   
   \paragraph{Base case} 
   For $h = 0$, we have $\epsilon \geq \barepsilon$ and, by assumption, $\epsilon \leq \barepsilon$. Therefore, $\epsilon = \barepsilon$.
   It can be easily verified that $B(\barepsilon) = \nsamplev(\barepsilon,\algdelta)$, hence the lemma is true for $h=0$.
   
   \paragraph{Induction hypothesis} Assume that the lemma is true for $h$.
   
   \paragraph{Induction step} Let $\epsilon \geq \barepsilon \sqrt{\gamma}^{h+1}$. We have that
   \begin{align*}
      \nsamplev(\epsilon,\algdelta) & = \underbrace{1}_{\mbox{current call}} + \underbrace{\nsamplev\pa{\frac{\epsilon}{\sqrt{\gamma}}, \algdelta}}_{\mbox{call in line 11 of $\estV$}} + \underbrace{K N(\sqrt{\barepsilon\epsilon}) \nsamplev\pa{ \sqrt{\frac{\barepsilon\epsilon}{\gamma}} , \algdelta}}_{\mbox{calls in $\estQ$}} \\
      & = 1 + \nsamplev\pa{\frac{\epsilon}{\sqrt{\gamma}}, \algdelta} + \frac{\beta(\algdelta)}{\epsilon} \nsamplev\pa{ \sqrt{\frac{\barepsilon\epsilon}{\gamma}}, \algdelta } \\
      & \leq \nsamplev\pa{\frac{\epsilon}{\sqrt{\gamma}}, \algdelta} + \frac{2\beta(\algdelta)}{\epsilon} \nsamplev\pa{ \sqrt{\frac{\barepsilon\epsilon}{\gamma}}, \algdelta }
   \end{align*}
   Since $\epsilon \geq \barepsilon \sqrt{\gamma}^{h+1}$ and  $\epsilon \leq \barepsilon$, we have $ \sqrt{\frac{\barepsilon\epsilon}{\gamma}} \geq \frac{\epsilon}{\sqrt{\gamma}} \geq \barepsilon\sqrt{\gamma}^h$. This allows us to use our induction hypothesis to get
   \begin{align*}
      \nsamplev(\epsilon,\algdelta) \leq B\pa{\frac{\epsilon}{\sqrt{\gamma}}} + \frac{2\beta(\algdelta)}{\epsilon} B\pa{ \sqrt{\frac{\barepsilon\epsilon}{\gamma}} }\cdot
   \end{align*}
   We will need the equation bellow, which is easily verified as 
   \begin{align}
      & \log\pa{ \frac{\barepsilon}{\sqrt{\frac{\barepsilon\epsilon}{\gamma}} \gamma}  } = \frac{1}{2} \log\pa{\frac{\barepsilon/\gamma}{\epsilon}}
   \end{align}
   We have that
   \begin{align*}
      \frac{B\pa{\frac{\epsilon}{\sqrt{\gamma}}}}{B(\epsilon)} & = \frac{B_1\pa{\frac{\epsilon}{\sqrt{\gamma}}}}{B_1(\epsilon)}\frac{B_2\pa{\frac{\epsilon}{\sqrt{\gamma}}}}{B_2(\epsilon)} \\
      & = \gamma {\underbrace{\left[ \frac{ \log\pa{\frac{\barepsilon/\gamma}{\epsilon}} - \frac{1}{2}\log\frac{1}{\gamma} } {  \log\pa{\frac{\barepsilon/\gamma}{\epsilon}}   } \right]}_{< 1}}^{ \eta_2(\algdelta) } \\
      & \leq \gamma, 
   \end{align*}
   where we used the assumption that $\eta_2(\algdelta) \geq 0$.
   
   Also we get that
   \begin{align*}
      \frac{B\pa{ \sqrt{\frac{\barepsilon\epsilon}{\gamma}} }}{B(\epsilon)} & = \frac{\epsilon\gamma}{\barepsilon} \frac{B_1\pa{ \sqrt{\frac{\barepsilon\epsilon}{\gamma}} }}{B_1(\epsilon)} \\
      & = \frac{\epsilon\gamma}{\barepsilon} \left[ \frac{ \log_{\frac{1}{\gamma}}\pa{ \frac{\barepsilon}{\sqrt{\frac{\barepsilon\epsilon}{\gamma}} \gamma}} } {  \log_{\frac{1}{\gamma}}\pa{\frac{\barepsilon/\gamma}{\epsilon}}   } \right]^{\eta_2(\algdelta)} \\ 
       & = \frac{\epsilon\gamma}{\barepsilon} \left[ \frac{ \frac{1}{2}\log_{\frac{1}{\gamma}}\pa{ \frac{\barepsilon/\gamma}{\epsilon}} } {  \log_{\frac{1}{\gamma}}\pa{\frac{\barepsilon/\gamma}{\epsilon}}   } \right]^{\eta_2(\algdelta)} \\ 
       & = \frac{\epsilon\gamma}{\barepsilon}  \pa{\frac{1}{2}}^{\eta_2(\algdelta)}
       = \frac{\epsilon\gamma}{\barepsilon} \frac{(1-\gamma)}{\gamma}  \frac{\barepsilon}{2\beta(\algdelta)} = \frac{(1-\gamma)\epsilon}{2\beta(\algdelta)}
   \end{align*}
   Finally, we obtain
   \begin{align*}
      \nsamplev(\epsilon,\algdelta) & \leq B\pa{\frac{\epsilon}{\sqrt{\gamma}}} + \frac{2\beta(\algdelta)}{\epsilon} B\pa{ \sqrt{\frac{\barepsilon\epsilon}{\gamma}} } \\
      & \leq \gamma B(\epsilon) + \frac{2\beta(\algdelta)}{\epsilon}\frac{(1-\gamma)\epsilon}{2\beta(\algdelta)}B(\epsilon) \\
      & = B(\epsilon),
   \end{align*}
   which proves the lemma. 
\end{proof}
Now we can prove Theorem~\ref{thm:sample-complexity}, which is restated below. 
\begin{theorem*}
   Let $\samcomplex{\epsilon}{\algdelta}$ be the number of calls to the generative model (oracle) before the algorithm terminates. For any state $\svar\in\states$ and $\epsilon, \algdelta > 0$, 
   
   \begin{align*}
   \samcomplex{\epsilon}{\algdelta} & \leq \frac{c_1}{\epsilon^4}\log\pa{\frac{c_2}{\algdelta}} \left[ c_3\log\pa{\frac{c_4}{\epsilon}} \right]^{ \log_2\pa{c_5 \pa{ \log\pa{\frac{c_2}{\algdelta}}}}} 
   = \cOtilde{ \frac{1}{\epsilon^{4}}}
   \end{align*}
   where $c_1, c_2, c_3, c_4$ and $c_5$ are constants that depend only on $K$, $L$ and $\gamma$. 
\end{theorem*}
\begin{proof}
   First, notice that the number of calls to the generative model is smaller than the total number of calls to $\estV$. $\ouralgo$ makes one call to $\estQ$, which makes $N(\epsilon)$ calls to $\estV$. If $\epsilon\geq\barepsilon$, Proposition \ref{prop:sample-complexity-large-epsilon} shows that the sample complexity is bounded by a constant. Lemma \ref{lemma:sample-complexity} bounds the sample complexity for $\epsilon \leq \barepsilon$, and we use it to bound $\samcomplex{\epsilon}{\algdelta}$:
   \begin{align*}
      \samcomplex{\epsilon}{\algdelta} &= N(\epsilon) \nsamplev(\epsilon, \algdelta) \\
      & \leq  N(\epsilon) \eta_1 \left[ \log_{\frac{1}{\gamma}}\pa{\frac{\barepsilon/\gamma}{\epsilon}} \right]^{\eta_2(\algdelta)} \frac{1}{\epsilon^2}\\
      & \leq \frac{c_1}{\epsilon^4}\log\pa{\frac{c_2}{\algdelta}} \left[ c_3\log\pa{\frac{c_4}{\epsilon}} \right]^{ \log_2\pa{c_5 \pa{ \log\pa{\frac{c_2}{\algdelta}}}}} 
      = \cOtilde{ \frac{1}{\epsilon^{4}}} 
   \end{align*}
   by using the definition of $N(\epsilon)$ for $\epsilon \leq \barepsilon$ and the definition of $\eta_2(\algdelta)$ in Lemma \ref{lemma:sample-complexity}. 

    The constants are given by:
    
    \begin{itemize}
        \item $c_1 = \frac{18 (1+\supFsZero)^2 \nsamplev(\barepsilon, \algdelta) }{K^2L^2(1-\gamma)^4}$; 
        \item $c_2 = 2K$;
        \item $c_3 = [\log\pa{1/\gamma}]^{-1}$;
        \item $c_4 = (1-\sqrt{\gamma})/(\gamma KL)$;
        \item $c_5 = \frac{36(1+\supFsZero)^2 \gamma K^3 L^2}{(1-\gamma)^5 (1-\sqrt{\gamma})^4}$.
    \end{itemize}
\end{proof}
\section{Consistency}
\thmpac*
To prove that our algorithm outputs a good estimate of the value function with high probability, we proceed as follows:

\begin{itemize}
   \item In Lemma \ref{lemma:consistency}, we prove that the output of $\estV$, conditioned on an event $\event$, is a low-bias estimate of the true value function, and that $\event$ happens with high probability;
   \item Given  Lemma \ref{lemma:consistency}, the proof of Theorem \ref{thm:consistency} is straightforward.
\end{itemize}

Throughout the proof, we will make distinctions between two cases:

\begin{itemize}
   \item {\bf Case 1:} $\barepsilon \leq \epsilon < \frac{1+M_\lambda}{1-\gamma} $
   \item {\bf Case 2:} $\epsilon < \barepsilon$
\end{itemize}

\subsection{Definitions}

We define the function $\zeta(\epsilon)$ as 

\begin{equation}
   \zeta(\epsilon) = %
   \begin{cases}
      \epsilon, \quad & \mbox{ if } \barepsilon \leq \epsilon < \frac{1+M_\lambda}{1-\gamma}, \\
      \sqrt{\barepsilon\epsilon}, \quad & \mbox{ if } \epsilon < \barepsilon, \\
      \infty, \quad & \mbox{ otherwise.}
   \end{cases}
\end{equation}

Define $\params(s, \epsilon)$ as the (random) set of parameters used to call $\estV$ after a call to $\estV(s,\epsilon)$, that is

\begin{align}
   \params(s, \epsilon) = \left\lbrace \left(Z_{s,a}^{(k)}, \frac{\zeta(\epsilon)}{\sqrt{\gamma}} \right) \mbox{  for  } k=1,\ldots,N(\epsilon) ; a \in \actions \right\rbrace
\end{align}
in case 1 and 

\begin{align}
\params(s, \epsilon) = \left\lbrace \left(Z_{s,a}^{(k)}, \frac{\zeta(\epsilon)}{\sqrt{\gamma}} \right) \mbox{  for  } k=1,\ldots,N(\epsilon) ; a \in \actions \right\rbrace \bigcup \left\lbrace  \left(Z_{s,A}, \frac{\epsilon}{\sqrt{\gamma}} \right) \right\rbrace 
\end{align}
in  case 2, where $Z_{s,a}^{(k)}$ are the next states sampled in $\estQ$ and $Z_{s,A}$ is the next state sampled $\estV(s,\epsilon)$. 

A call to $\estV(s, \epsilon)$ makes one call to $\estQ$. Denote the output of this call to $\estQ$ by $\QFuncHat_s^\epsilon$. We define the event $\event(s, \epsilon)$ as follows:

\begin{align}
    \event(s, \epsilon) =  %
    \begin{cases}
      & \left\lbrace \norm{\QFuncHat_s^\epsilon - \QFunc_s }_\infty \leq \zeta(\epsilon) \right\rbrace \bigcap \subevent(s, \epsilon),\quad \mbox{ if } 0 < \epsilon < \frac{1+M_\lambda}{1-\gamma}, \\
      & \Omega, \quad \mbox{ if } \epsilon \geq  \frac{1+M_\lambda}{1-\gamma}.   
   \end{cases}
\end{align}  
where $\Omega$ is the whole sample space and 

\begin{align}
 \subevent(s, \epsilon) = \bigcap_{ (z,e) \in \params(s,\epsilon)} \event(z, e) 
\end{align}

Define $C_\gamma$ as:

\begin{align}
    C_\gamma = \frac{3(1 + 
    \supFsZero)}{(1-\gamma)^2}
\end{align}


\subsection{Proofs}

\begin{lemma}
   \label{lemma:consistency}
   Let $\ValueFuncH_\epsilon(s) = \estV(\svar,\epsilon)$. For all $h\in\N, \svar\in\states, \epsilon \ge \frac{(1+M_\lambda)\sqrt{\gamma}^h}{1-\gamma} \CommaBin$, we have: 
   \begin{align*}
   \mathbf{(i)} \quad &\abs{\EE{\ValueFuncH_\epsilon(s) \Big| \event(s,\epsilon)} - \Value{s}} \le \epsilon, \text{ and}\\
   \mathbf{(ii)}  \quad & \PP{ \abs{\ValueFuncH_\epsilon(s)} \leq  C_\gamma  \Big| \event(s,\epsilon) }  = 1 \\
   \mathbf{(iii)} \quad & \PP{\event(s,\epsilon) } \geq 1 - \algdelta \nsamplev(\epsilon, \algdelta)
   \end{align*}
   where
    
    \begin{align}
        \nsamplev(\epsilon, \algdelta) = 1 + \sum_{(z,e) \in \params(s,\epsilon)} \nsamplev(e, \algdelta)
    \end{align}

     is the total number of recursive calls to $\estV$  after an initial call with parameters $(s, \epsilon)$.
\end{lemma}

\begin{proof}
   We proceed by induction over $h$. 
   
   \paragraph{(1) Base case.} If $h = 0$, $\epsilon \geq \frac{1+M_\lambda}{1-\gamma}$ and $\event(s,\epsilon) = \Omega$. The output is then $\ValueFuncH_\epsilon(s) = 0$. Point {\bf (i)} is verified by using the fact that $\abs{\Value{s}} \leq \frac{1+M_\lambda}{1-\gamma} \leq \epsilon$; points {\bf (ii)} and {\bf (iii)} are trivially verified. 
   
   \paragraph{(2) Induction hypothesis.} Assume that {\bf (i)}, {\bf (ii)} and {\bf (iii)} are true for $h$.
   
   \paragraph{(3) Induction step.} Let $\epsilon  \geq  \frac{(1+M_\lambda)\sqrt{\gamma}^{h+1}}{1-\gamma}$. This implies that $\epsilon/\sqrt{\gamma}$ and $\zeta(\epsilon)/\sqrt{\gamma}$ are both greater than $ \frac{(1+M_\lambda)\sqrt{\gamma}^h}{1-\gamma}$, which will allow us to use our induction hypothesis. 
   
   We start by proving {\bf (iii)}.
   
   Let $\QFuncHat_s^\epsilon = \estQ(s,\zeta(\epsilon))$. Let the reward $R_{s,a}^{(k)}$ and state $Z_{s,a}^{(k)}$ be the random variables associated to the $k$-th call to the generative model used to compute $\QFuncHat_s$ in $\estQ$, for $k \in \{1, \cdots, N(\epsilon)\}$. Let 
   
   \begin{align}
   q_s^k(a) := R_{s,a}^{(k)} + \gamma  \estV\pa{Z_{s,a}^{(k)}, \zeta(\epsilon)/\sqrt{\gamma}}
   \end{align}
   and let
   
   \begin{align}
   \label{eq:estQ_output}
   \QFuncBar_s^\epsilon(a) = \frac{1}{N(\epsilon)}\sum_{k=1}^{N(\epsilon)}q_s^k(a)
   \end{align}

    so that:
    
    \begin{align}
        \QFuncHat_s^\epsilon = \clip{(1+M_{\lambda})(1-\gamma)^{-1}}{\QFuncBar_s^\epsilon(a)}
    \end{align}

    Using Fact \ref{fact:clipping}, we have:
   \begin{align}
   \abs{ \QFuncHat_s^\epsilon(a) - \QFunc_s(a)  } & \leq \abs{ \QFuncBar_s^\epsilon(a) - \QFunc_s(a)  } \\
   & \leq \underbrace{ \abs{\QFuncBar_s^\epsilon(a) - \EE{\QFuncBar_s^\epsilon(a) | \subevent(s, \epsilon) }}}_{ {\bf (I)} } + \underbrace{ \abs{\EE{\QFuncBar_s^\epsilon(a) | \subevent(s, \epsilon) }-\QFunc_s(a)}}_{ {\bf (II)} }
   \end{align}

We'd like to use Hoeffding's inequality to bound {\bf (I)} in probability. For that, we need to verify that the random variables $\{q_s^k(a)\}_{k=1}^{N(\epsilon)}$ are bounded and independent conditionally on $\subevent(s, \epsilon)$. 
   
\textbf{Boundedness.}  By induction hypothesis {\bf (ii)} In the event $\subevent(s,\epsilon)$, the random variables $\estV\pa{Z_{s,a}^{(k)}, \zeta(\epsilon)/\sqrt{\gamma}}$, for all $k$, are bounded by $C_\gamma$. Using the fact that the rewards are in $[0, 1]$ and that $C_\gamma \geq 1/(1-\gamma)$, we obtain $q_s^k(a)$ is also bounded by $C_\gamma$.

\textbf{Independence.} Let $E_k = \event\pa{Z_{s,a}^{k}, \zeta(\epsilon)/\sqrt{\gamma}}$. For any $t \in \R^{N(\epsilon)}$, the characteristic function of $\{q_s^k(a)\}_{k=1}^{N(\epsilon)}$ conditionally on $\subevent(s, \epsilon)$ is given by

\begin{align*}
    \EE{ \exp\pa{\imag \sum_k t_k q_s^k(a)} \given \subevent(s, \epsilon) } 
    & \overeq{(a)} \EE{ \exp\pa{\imag \sum_k t_k q_s^k(a)} \given \bigcap_k E_k } \\
    & = \frac{ \EE{ \exp\pa{\imag \sum_k t_k q_s^k(a)} \prod_k \indic{E_k}  } }{ \EE{ \prod_k \indic{E_k}  } } \\
    & = \frac{ \EE{  \prod_k \exp\pa{\imag t_k q_s^k(a)} \indic{E_k}  } }{ \EE{ \prod_k \indic{E_k}  } } \\
    & \overeq{(b)} \frac{ \prod_k  \EE{  \exp\pa{\imag t_k q_s^k(a)} \indic{E_k}  } }{  \prod_k \EE{ \indic{E_k}  } } \\
    & = \prod_k  \EE{  \exp\pa{\imag t_k q_s^k(a)} \given E_k } \\
    & \overeq{(c)} \prod_k  \EE{  \exp\pa{\imag t_k q_s^k(a)} \given \subevent(s, \epsilon) }
\end{align*}
which is justified by

\begin{enumerate}[(a)]
    \item Definition of $\subevent(s, \epsilon)$ and the fact that $\{q_s^k(a)\}_{k=1}^{N(\epsilon)}$ are independent of $\event\pa{Z_{s,A}, \frac{\epsilon}{\sqrt{\gamma}}}$;
    \item The random variables $\{q_s^k(a)\}_{k=1}^{N(\epsilon)}$ are independent and the events $\{E_k\}_{i=1}^{N(\epsilon)}$ are also independent; 
    \item The random variable $q_s^k(a)$ is independent of every $E_j$ for $j \neq k$.  
\end{enumerate}

Since the characteristic function of $\{q_s^k(a)\}_{k=1}^{N(\epsilon)}$ is the product of their characteristic functions, these random variables are independent given $\subevent(s, \epsilon)$. 


Now we can use Hoeffding's inequality:

   \begin{align*}
   & \PP{\abs{\QFuncBar_s^\epsilon(a) - \EE{\QFuncBar_s^\epsilon(a) \Big| \subevent(s, \epsilon) }} \geq (1-\sqrt{\gamma})\zeta(\epsilon) \Big| \subevent(s, \epsilon) } \\ 
   & = \PP{\abs{\frac{1}{N(\epsilon)}\sum_{k=1}^{N(\epsilon)}q_s^k(a) - \EE{q_s^k(a)\Big| \subevent(s, \epsilon)} } \geq (1-\sqrt{\gamma})\zeta(\epsilon) \Big| \subevent(s, \epsilon) }  \\
   & \leq 2\exp\pa{-  \frac{ N(\epsilon) (1-\sqrt{\gamma})^2 \zeta(\epsilon)^2}{ 2 C_\gamma^2  }  } \\
   & \leq \frac{\algdelta}{K}
   \end{align*}

   And {\bf (II)} is bounded by using the induction hypothesis {\bf (i)}:

   \begin{align*}
      & \abs{\EE{q_s^k(a)\Big| \subevent(s, \epsilon)}-\QFunc_s(a)} \\
      & \overeq{(a)} \gamma \abs{ \EE{\estV\pa{Z_{s,a}^{(k)}, \frac{\zeta(\epsilon)}{\sqrt{\gamma}}}\Big| \subevent(s, \epsilon)} - \EE{\ValueFunc(Z_{s,a}^{(k)})\Big| \subevent(s, \epsilon)} } \\
      & \overeq{(b)} \gamma \abs{ \EE{\estV\pa{Z_{s,a}^{(k)}, \frac{\zeta(\epsilon)}{\sqrt{\gamma}}}\Big| \event\pa{Z_{s,a}^{(k)}, \frac{\zeta(\epsilon)}{\sqrt{\gamma}}}} - \EE{\ValueFunc(Z_{s,a}^{(k)})\Big| \event\pa{Z_{s,a}^{(k)}, \frac{\zeta(\epsilon)}{\sqrt{\gamma}}}} } \\
      & \overeq{(c)} \gamma \abs{  \EE{ \EE{ \estV\pa{Z_{s,a}^{(k)}, \frac{\zeta(\epsilon)}{\sqrt{\gamma}}}  \Big| Z_{s,a}^{(k)}, \event\pa{Z_{s,a}^{(k)}, \frac{\zeta(\epsilon)}{\sqrt{\gamma}}}}  - \ValueFunc(Z_{s,a}^{(k)}) \Big| \event\pa{Z_{s,a}^{(k)}, \frac{\zeta(\epsilon)}{\sqrt{\gamma}}}    }  } \\
      & \overleq{(d)} \gamma \frac{\zeta(\epsilon)}{\sqrt{\gamma}} \\
      & = \sqrt{\gamma}\zeta(\epsilon)
   \end{align*}
   which is justified by the following: 
   \begin{enumerate}[(a)]
      \item $\EE{R_{s,a}^{(k)} \Big| \subevent(s, \epsilon) } = \EE{R_{s,a}^{(k)}}$, since the reward depends only on $s, a$;
      \item The term  $\pa{Z_{s,a}^{(k)}, \frac{\zeta(\epsilon)}{\sqrt{\gamma}}}$ depends on $\subevent(s, \epsilon)$ only through $\event\pa{Z_{s,a}^{(k)}, \frac{\zeta(\epsilon)}{\sqrt{\gamma}}}$;
      \item Law of total expectation;
      \item Consequence of induction hypothesis {\bf (i)}.
   \end{enumerate}
   
   Putting together the bounds for {\bf (I)} and {\bf (II)} and doing an union bound over all actions, we obtain: 
   
   \begin{align*}
      & \PP{ \norm{ \QFuncHat_s^\epsilon - \QFunc_s }_\infty \geq \zeta(\epsilon)  \Big| \subevent(s, \epsilon)} \leq \algdelta
   \end{align*}

   We can now give a lower bound to the probability of the event $\event(s, \epsilon)$. Let 
   
   \begin{align}
      \auxevent = \left\lbrace \norm{ \QFuncHat_s^\epsilon - \QFunc_s }_\infty < \zeta(\epsilon) \right\rbrace 
   \end{align}
   
   We have:
   
   \begin{align*}
      \PP{\event(s, \epsilon)} & \geq \PP{ \auxevent \cap \subevent(s,\epsilon) } \\
      & = \PP{\auxevent \Big|  \subevent(s,\epsilon)} \PP{\subevent(s,\epsilon)} \\
      & = \pa{1 - \PP{\auxevent^\complement \Big|  \subevent(s,\epsilon)} }\PP{\subevent(s,\epsilon)} \\
      & \geq \PP{\subevent(s,\epsilon)} - \algdelta \\
      & \geq 1  - \algdelta \nsamplev(\epsilon, \algdelta)
   \end{align*}
   since 
   
   \begin{align*}
      \PP{\subevent(s,\epsilon)} & = 1 - \PP{\subevent(s,\epsilon)^\complement} \\
      & = 1 - \PP{\bigcup_{(z,e) \in \params(s,\epsilon)} \event(z,e)^\complement } \\
      & \geq 1 - \sum_{(z,e) \in \params(s,\epsilon)} \PP{\event(z,e)^\complement} \\
      & \geq 1 - \algdelta \sum_{(z,e) \in \params(s,\epsilon)} \nsamplev(e, \algdelta) \quad \mbox{  by induction hypothesis (iii)} \\
      & = 1 - \algdelta ( \nsamplev(\epsilon, \algdelta) - 1)
   \end{align*}

    This proves {\bf (iii)}. Now, let's prove {\bf (i)}. 
   
   For any event $\auxevent$, we write
   
   \begin{align*}
   \EE[\auxevent]{\cdot} = \EE{ \cdot \Big|\auxevent}
   \end{align*}
   
   
   \paragraph{Case 1. }We start with case 1, $\barepsilon \leq \epsilon < \frac{1+M_\lambda}{1-\gamma}$, where $\zeta(\epsilon) = \epsilon$ and 
   
   \begin{align}
      \ValueFuncH_\epsilon(s) = \stateFunc{s}(\QFuncHat_s^\epsilon)
   \end{align}
   
   We have:
   
   \begin{align*}
   \abs{\EE[\event(s,\epsilon)]{\ValueFuncH_\epsilon(s)} - \ValueFunc(s)} & = \abs{ \EE[\event(s,\epsilon)]{\stateFunc{s}(\QFuncHat_s^\epsilon) - \stateFunc{s}(\QFunc_s) } } \\
   & \leq \EE[\event(s,\epsilon)]{ \abs{ \stateFunc{s}(\QFuncHat_s^\epsilon) - \stateFunc{s}(\QFunc_s)  } } \\
   & \leq \EE[\event(s,\epsilon)]{ \norm{\QFuncHat_s^\epsilon(a) - \QFunc_s(a)}_\infty } \\
   & \leq \zeta(\epsilon) = \epsilon
   \end{align*}
   and {\bf(i)} is verified for case 1. 
   
   \paragraph{Case 2. }Consider now the case 2, $\epsilon < \barepsilon$, where $\zeta(\epsilon) = \sqrt{\barepsilon\epsilon}$.
   
   Let $A$ be the action following the distribution $\frac{\nabla\stateFunc{\svar}\pa{\QFuncHat_s^\epsilon}}{\norm{\nabla\stateFunc{\svar}\pa{\QFuncHat_s^\epsilon}}_1}$, and let the reward $R_{s,A}$ and the state $Z_{s,A}$ be the random variables associated to the call to the generative model with parameters $(s, A)$. Let $\hat{v} = \estV\pa{Z_{s,A}, \epsilon/\sqrt{\gamma}}$.  The output in this case is given by
   
   \begin{align}
   \ValueFuncH_\epsilon(s) = \stateFunc{\svar}\pa{\QFuncHat_s^\epsilon } - (\QFuncHat_s^\epsilon)\transpose\nabla\stateFunc{\svar}\pa{\QFuncHat_s^\epsilon } + (\rewardrand + \gamma\hat{v}) \norm{\nabla\stateFunc{\svar}\pa{\QFuncHat_s^\epsilon }}_1
   \end{align}

   Let 
   \begin{align*}
   \QFunc_s(A) & = \EE[\event(s,\epsilon)]{R_{s,A} + \gamma \ValueFunc(Z_{s,A}) | A, \QFuncHat_s^\epsilon} \\
   & = \EE[\event(s,\epsilon)]{R_{s,A} + \gamma \ValueFunc(Z_{s,A}) | A} \quad
   \end{align*}
   and let
   
   \begin{align}
   \Vtilde(s) = \EE[\event(s,\epsilon)]{\stateFunc{\svar}\pa{\QFuncHat_s^\epsilon } -(\QFuncHat_s^\epsilon) \transpose\nabla\stateFunc{\svar}\pa{\QFuncHat_s^\epsilon } + \QFunc_s(A) \norm{\nabla\stateFunc{\svar}\pa{\QFuncHat_s^\epsilon }}_1}  
   \end{align}
   
   We have
   
   \begin{align*}
   & \abs{ \EE[\event(s,\epsilon)]{\ValueFuncH_\epsilon(s)} -  \Vtilde(s)} \\
   & \overeq{(a)} \gamma \abs{\EE[\event(s,\epsilon)]{ \EE[\event(s,\epsilon)]{ \estV\pa{Z_{s,A}, \frac{\epsilon}{\sqrt{\gamma}}} -  \ValueFunc(Z_{s,A}) \Big| A, \QFuncHat_s^\epsilon, Z_{s,A}} \norm{\nabla\stateFunc{\svar}\pa{\QFuncHat_s^\epsilon }}_1}} \\
   & \overeq{(b)}  \gamma \abs{\EE[\event(s,\epsilon)]{ \pa{\EE[\event(s,\epsilon)]{ \estV\pa{Z_{s,A}, \frac{\epsilon}{\sqrt{\gamma}}} \Big| A, \QFuncHat_s^\epsilon, Z_{s,A}} -  \ValueFunc(Z_{s,A}) }} \norm{\nabla\stateFunc{\svar}\pa{\QFuncHat_s^\epsilon }}_1} \\
   & \overleq{(c)}  \gamma \EE[\event(s,\epsilon)]{ \abs{\EE[\event(s,\epsilon)]{ \estV\pa{Z_{s,A}, \frac{\epsilon}{\sqrt{\gamma}}}  \Big|A, \QFuncHat_s^\epsilon, Z_{s,A}} -  \ValueFunc(Z_{s,A}) }}  \\
   & \overeq{(d)}  \gamma \EE[\event(s,\epsilon)]{ \abs{\EE[\event(Z_{s,A},\epsilon/\sqrt{\gamma})]{ \estV\pa{Z_{s,A}, \frac{\epsilon}{\sqrt{\gamma}}}  \Big|Z_{s,A}} -  \ValueFunc(Z_{s,A}) }}  \\
   & \overleq{(e)} \gamma \frac{\epsilon}{\sqrt{\gamma}} = \sqrt{\gamma}\epsilon
   \end{align*}
   which is justified by the following points:
   
   \begin{enumerate}[(a)]
      \item The reward depend only on $s, a$ and law of total expectation;
      \item $\ValueFunc(Z_{s,A})$ is a function of $Z_{s,A}$ and no other random variable;
      \item Jensen's inequality and the fact that $\normm{\nabla\stateFunc{\svar}\pa{\QFuncHat_s^\epsilon }}_1 \leq 1$;
      \item Given $Z_{s, A}$, the term $\estV\pa{Z_{s,A}, \frac{\epsilon}{\sqrt{\gamma}}}$ depends on $\event(s,\epsilon)$ only through $\event(Z_{s,A},\epsilon/\sqrt{\gamma})$;
      \item Induction hypothesis {\bf (i)}.
   \end{enumerate}

   Now, $\EE[\event(s,\epsilon)]{\QFunc_s(A)\norm{\nabla\stateFunc{\svar}\pa{\QFuncHat_s^\epsilon }}_1 }$ can be written as
   
   \begin{align*}
   &  \EE[\event(s,\epsilon)]{\QFunc_s(A)\norm{\nabla\stateFunc{\svar}\pa{\QFuncHat_s^\epsilon }}_1 } \\
   & = \EE[\event(s,\epsilon)]{\EE[\event(s,\epsilon)]{\QFunc_s(A) \Big| \QFuncHat_s^\epsilon } \norm{\nabla\stateFunc{\svar}\pa{\QFuncHat_s^\epsilon }}_1 } \\
   & = \EE[\event(s,\epsilon)]{ \QFunc_s\transpose\nabla\stateFunc{\svar}\pa{\QFuncHat_s^\epsilon }  }
   \end{align*}
   so that $\Vtilde(s)$ is given by
   
   \begin{align}
   \Vtilde(s) = \EE[\event(s,\epsilon)]{\stateFunc{\svar}\pa{\QFuncHat_s^\epsilon } + (\QFunc_s - \QFuncHat_s^\epsilon) \transpose\nabla\stateFunc{\svar}\pa{\QFuncHat_s^\epsilon }}
   \end{align}
   
   Finally, we bound the difference between $\Vtilde(s)$ and $\ValueFunc(s)$:
   
   \begin{align*}
   \abs{ \Vtilde(s) - \ValueFunc(s)} & \leq \EE[\event(s,\epsilon)]{\abs{\stateFunc{\svar}\pa{\QFuncHat_s^\epsilon } + (\QFunc_s - \QFuncHat_s^\epsilon) \transpose\nabla\stateFunc{\svar}\pa{\QFuncHat_s^\epsilon } - \ValueFunc(s) }} \\ 
   & \leq L \EE[\event(s,\epsilon)]{\norm{\QFunc_s - \QFuncHat_s^\epsilon}_2^2} \\
   & \overleq{(a)} KL \EE[\event(s,\epsilon)]{\norm{\QFunc_s - \QFuncHat_s^\epsilon}_\infty^2} \\
   & \leq KL \zeta(\epsilon)^2 \\
   & = KL \barepsilon \epsilon \\
   & = (1-\sqrt{\gamma})\epsilon
   \end{align*}
   by using the fact that we are on $\event(s, \epsilon)$ and (a) uses the fact that for all $x \in \R^K$, $\norm{x}_2^2 \leq K \norm{x}_\infty^2$. 
   
   We can now prove {\bf (i)} for case 2:
   
   \begin{align}
      \abs{ \EE[\event(s,\epsilon)]{\ValueFuncH_\epsilon(s)} -  \ValueFunc(s)} & \leq \abs{ \EE[\event(s,\epsilon)]{\ValueFuncH_\epsilon(s)} -  \Vtilde(s)} + \abs{ \Vtilde(s) - \ValueFunc(s)} \\
      & \leq \sqrt{\gamma}\epsilon + (1-\sqrt{\gamma})\epsilon = \epsilon 
   \end{align}

   Finally, let's prove {\bf (ii)}.

   \paragraph{Case 1.} In this case, $\ValueFuncH_\epsilon(s) = \stateFunc{\svar}(\QFuncHat_s^\epsilon)$ with $\normm{\QFuncHat_s^\epsilon}_\infty \leq (1+\supFsZero)/(1-\gamma)$, since each component of $\QFuncHat_s^\epsilon$ is clipped and lie in the interval $\left[0,\frac{1+\supFsZero}{1-\gamma}\right]$. The assumptions on $\stateFunc{\svar}$ imply that $\abs{\ValueFuncH_\epsilon(s)} \leq \frac{1+\supFsZero}{1-\gamma} \leq C_\gamma$.
   
   
   

   \paragraph{Case 2.} In this case, we have:

    \begin{align*}
   \abs{\ValueFuncH_\epsilon(s)} & \leq \abs{\stateFunc{\svar}\pa{\QFuncHat_s^\epsilon } - (\QFuncHat_s^\epsilon)\transpose\nabla\stateFunc{\svar}\pa{\QFuncHat_s^\epsilon }} + \abs{\rewardrand + \gamma\hat{v}}\norm{\nabla\stateFunc{\svar}\pa{\QFuncHat_s^\epsilon }}_1  \\
   &  \leq 2\norm{\QFuncHat_s^\epsilon}_\infty + \supFsZero + 1 + \gamma C_\gamma \\
   & \leq \frac{2(1 +\supFsZero)}{1-\gamma}+ \supFsZero + 1 + \gamma C_\gamma \\
   & \leq C_\gamma
    \end{align*}  
    since $\abs{\hat{v}} \leq C_\gamma$ by induction hypothesis {\bf (ii)}.

This proves {\bf (ii)} for case 2:

   \begin{equation}
      \PP{ \abs{\ValueFuncH(s)} \leq C_\gamma  \Big| \event(s,\epsilon) }  = 1
   \end{equation}
\end{proof}

Now, we can prove Theorem \ref{thm:consistency}, which is restated as follows:

\begin{theorem*}
   Let $\ValueH{\svar}$ be the output of $\ouralgo(s, \epsilon, \algdelta)$. For any state $\svar\in\states$ and $\epsilon, \algdelta > 0$, 
   \[\PP{\abs{\ValueH{\svar} - \Value{\svar}} > \epsilon} \le \algdelta\samcomplex{\epsilon}{\algdelta}.\] 
\end{theorem*}
\begin{proof}
   Let $\QFuncHat_s = \estQ(s, \epsilon)$. We have $\ValueFuncH(s) = \stateFunc{s}(\QFuncHat_s)$. As in the proof of Lemma \ref{lemma:consistency},  let the reward $R_{s,a}^{(k)}$ and state $Z_{s,a}^{(k)}$ be the random variables associated to the $k$-th call to the generative model used to compute $\QFuncHat_s(a)$ in $\estQ$, for $k \in \{1, \cdots, N(\epsilon)\}$.
   
   We have:
   
   \begin{align}
      \QFuncHat_s(a) = \frac{1}{N(\epsilon)}\sum_{k=1}^{N(\epsilon)} R_{s,a}^{(k)} + \gamma  \estV\pa{Z_{s,a}^{(k)}, \epsilon/\sqrt{\gamma}}
   \end{align}
   
   Consider the event $\auxevent$ defined by:
   
   \begin{align}
      \auxevent = \bigcap_{k=1}^{N(\epsilon)} \event\pa{Z_{s,a}^{(k)}, \frac{\epsilon}{\sqrt{\gamma}}}
   \end{align}
   By the same arguments as in the proof of Lemma \ref{lemma:consistency}, we have:
   
   \begin{itemize}
      \item In $\auxevent$, we have $\normm{\QFuncHat_s - \QFunc_s}_\infty \leq \epsilon$;
      \item $\PP{\auxevent} \geq 1 - \algdelta N(\epsilon) \nsamplev(\epsilon,\algdelta) = 1 - \algdelta\samcomplex{\epsilon}{\algdelta}$.
   \end{itemize}

   This implies the result, since $\abs{\ValueH{\svar} - \Value{\svar}} \leq \norm{\QFuncHat_s - \QFunc_s}_\infty $.
   
   Now, for every $\epsilon > 0$ and every $\pacdelta > 0$, we need to be able to find a value of $\algdelta$ such that $\algdelta\samcomplex{\epsilon}{\algdelta} \leq \pacdelta$. That is, given $\epsilon$ and $\pacdelta$, we need to find $\algdelta$ such that 
   
    \begin{align}
        \algdelta \frac{c_1}{\epsilon^4}\log\pa{\frac{c_2}{\algdelta}} \left[ c_3\log\pa{\frac{c_4}{\epsilon}} \right]^{ \log_2\pa{c_5 \pa{ \log\pa{\frac{c_2}{\algdelta}}}}} \leq \pacdelta.
    \end{align}
    
    Such value exists, since the term on the LHS tends to $0$ as $\algdelta \to 0$, and it depends on $\epsilon$. We will show that this dependence is polynomial when $\epsilon \to 0$. 
    
    Let $\algdelta = \epsilon^5$. There exists a value $\tilde{\epsilon}$ that depends on $\pacdelta$ such that:
    
       \begin{align}
        \forall \epsilon \leq \tilde{\epsilon}, \quad \epsilon^5 \frac{c_1}{\epsilon^4}\log\pa{\frac{c_2}{\epsilon^5}} \left[ c_3\log\pa{\frac{c_4}{\epsilon}} \right]^{ \log_2\pa{c_5 \pa{ \log\pa{\frac{c_2}{\epsilon^5}}}}} \leq \pacdelta.
    \end{align}
    since the term on the LHS tends to $0$ as $\epsilon \to 0$, as a consequence of Proposition \ref{prop:useful-limit-1}. 
    
    Putting it all together, we can choose $\algdelta$ as folllows:
    
    \begin{align}
    \label{eq:delta-choice}
        \algdelta = %
        \begin{cases}
            \auxdelta \mbox{ such that } \auxdelta \frac{c_1}{\epsilon^4}\log\pa{\frac{c_2}{\auxdelta}} \left[ c_3\log\pa{\frac{c_4}{\epsilon}} \right]^{ \log_2\pa{c_5 \pa{ \log\pa{\frac{c_2}{\auxdelta}}}}} \leq \pacdelta, \quad \mbox{ if } \epsilon > \tilde{\epsilon}, \\ 
            \epsilon^5, \quad \mbox{ if } \epsilon \leq \tilde{\epsilon}
        \end{cases}
    \end{align}
    which is $\cO{\epsilon^5}$.

 Proposition \ref{prop:delta-eq-epsilon-5} implies that, for this choice of $\algdelta$, the sample complexity is still of order $\cO{ 1/\epsilon^{4+c}}$ for any $c > 0$.

\end{proof}

\section{Auxiliary results}

\begin{fact}
   \label{fact:fs_leq_max}
   For all $s \in \states$ and all $x \in \R^K$, we have $F_s(x) \leq \norm{x}_\infty + \sup_s |F_s(0)|$.
\end{fact}
\begin{proof}
   By the assumptions on $F_s$, we have:
   \begin{align}
   |F_s(x)| & = |F_s(x) - F_s(0) + F_s(0)| \leq|F_s(x) - F_s(0)| + |F_s(0)| \\
   & \leq \norm{x - 0}_\infty  + |F_s(0)| \leq  \norm{x}_\infty + \sup_s |F_s(0)|.
   \end{align}
\end{proof}


\begin{fact}
    \label{fact:clipping}
    Let $x, q \in \R^d$ be such that $0 \leq q_i \leq c$ for all $i$. Let $\tilde{x} = \clip{c}{x}$. Then, $\norm{\tilde{x}-q}_\infty \leq \norm{x-q}_\infty$.
\end{fact}
\begin{proof}
    For any $i\in\{1,\ldots, d\}$, we have $\abs{\tilde{x}_i - q_i} \leq \abs{x_i - q_i}$, since $0\le q_i\leq c$. The result follows.
\end{proof}

\begin{proposition}
   \label{prop:useful-limit-1}
   $\forall a,b,c >0$
   \begin{align*}
      \lim_{x\to\infty} \frac{1}{x^c}\exp\pa{ a [\log\log(x^b)]^2} = 0
   \end{align*}
\end{proposition}
\begin{proof}
   We have
   \begin{align*}
      \frac{1}{x^c}\exp\pa{ a [\log\log(x^b)]^2} & = \exp\pa{ a [\log\log(x^b)]^2 - c\log x }\\
      & = \exp\pa{ a [\log u ]^2 - \frac{c}{b} u  }, \quad \mbox{ by setting } u =  \log(x^b)
   \end{align*}
   And, for any $k > 0$, we have 
   
   \begin{align}
      \lim_{u\to\infty} \log^2 u - k u = -\infty.
   \end{align}
   which allows us to conclude. 
\end{proof}

\begin{proposition}
   \label{prop:delta-eq-epsilon-5}
   If we set $\algdelta = \algdelta(\epsilon) = \epsilon^5$, we have:
   \begin{align*}
   n(\epsilon, \algdelta(\epsilon)) = \mathcal{O}\pa{\frac{1}{\epsilon^{4+c}}}, \quad \forall c > 0
   \end{align*}
\end{proposition}

\begin{proof}
   We have:
   
   \begin{align*}
      \nsamplev(\epsilon, \algdelta(\epsilon)) & \leq \eta_1 \left[ \log_{\frac{1}{\gamma}}\pa{\frac{\bar{\epsilon}/\gamma}{\epsilon}} \right]^{\eta_2(\epsilon^3)} \frac{1}{\epsilon^2} \\
      & = \underbrace{\left[ \log_{\frac{1}{\gamma}}\pa{\frac{\bar{\epsilon}/\gamma}{\epsilon}} \right]^{\log_2\pa{k\log\pa{ \frac{2K}{\epsilon^3} }}}}_{ \mathrm{(A)}}  \frac{1}{\epsilon^2}
   \end{align*}
   where $k$ is a constant that does not depend on $\epsilon$. 
   The term (A) can be rewritten as: 
   
   \begin{align*}
      \left[ \log_{\frac{1}{\gamma}}\pa{\frac{\bar{\epsilon}/\gamma}{\epsilon}} \right]^{\log_2\pa{k\log\pa{ \frac{2K}{\epsilon^3} }}} & =  \left[c_1\log\pa{\frac{c_2}{\epsilon}}\right]^{ c_3 \log\left[ k \log\pa{\frac{c_4}{\epsilon^3}}  \right] }  \\
      & = \exp\left\lbrace  c_3 \log\left[ k \log\pa{\frac{c_4}{\epsilon^3}}\right] \log\pa{c_1\log\pa{\frac{c_2}{\epsilon}}}    \right\rbrace 
   \end{align*}
   which can be shown to be $\mathcal{O}\pa{\frac{1}{\epsilon^c}}$ for any $c > 0$ by applying proposition \ref{prop:useful-limit-1} after some algebraic manipulations.
   
   Hence,
   
   \begin{align*}
      \nsamplev(\epsilon, \algdelta(\epsilon)) = \frac{1}{\epsilon^2} \mathcal{O}\pa{\frac{1}{\epsilon^c}} = \mathcal{O}\pa{\frac{1}{\epsilon^{2+c}}}, \quad \forall c > 0.
   \end{align*}
   
   Since we have
   
   \begin{align*}
      \samcomplex{\epsilon}{\algdelta} &= N(\epsilon) \nsamplev(\epsilon, \algdelta)
   \end{align*}
   with $N(\epsilon) = \tilde{\mathcal{O}}\pa{1/\epsilon^2}$, this proves the result.
\end{proof}

\begin{corollary}
   \label{cor:delta-epsilon3}
   If we set $\algdelta = \algdelta(\epsilon) = \epsilon^5$, we have:
   \begin{align*}
   \lim_{\epsilon \to 0} \algdelta(\epsilon) n(\epsilon, \algdelta(\epsilon)) = 0
   \end{align*}
\end{corollary}
\begin{proof}
   It is an immediate consequence of proposition \ref{prop:delta-eq-epsilon-5} by taking $c \in ]0, 1[$.
\end{proof}

\section{On other smooth approximations of the max}
\label{sec:other-smooth-approx}

In this paper we focus on the $\logsumexp{\lambda}$ function as a smooth approximation to the maximum function. Yet our proof is more general and can handle any approximation of the max function which verifies the properties listed in Section~\ref{generalizations}. For instance let's consider the following regularization of the Bellman equation:
\begin{align}
\stateFunc{}(Q) &= \max_{(\pi_a)_{a\in\actions}} \sum_{a\in\actions} \pa{Q_a \cdot \pi_a + \lambda \sqrt{\pi_a}}  \end{align}

This smooth function is particularly interesting because it approximates the distribution of pulled armed of the UCB algorithm by taking $\lambda = 2c \cdot \sqrt{\frac{\ln(n)}{n}}$ (see \ref{UCB_like} and notice that $\pi_a^\star\cdot n$ approximates $n_a$). We show that this smooth approximation of the maximum verifies the assumptions made in Section~\ref{generalizations}. We have

\begin{align}
    \stateFunc{}(Q) &= \sum_{a\in\actions} \pa{Q_a \cdot \pi_a^\star + \lambda \sqrt{\pi_a^\star}}
\end{align}

and we can show that $\nabla_Q\stateFunc{}(Q) = \pi^\star$. Therefore point 1, 2 and 3 of Section~\ref{generalizations} are verified. Now by differentiating with respect to $\pi$ this time: 
\begin{align}
\label{UCB_like}
\forall a\in\actions \quad Q_a + \frac{\lambda}{2\sqrt{\pi_a^\star}} = U
\end{align}

where $U$ is the Lagrange multiplier. Using the fact that $\sum_{a\in\actions}{\pi_a^\star} = 1$, we get
\begin{align}
    \sum_{a\in\actions} \left(\frac{\lambda/2}{U - Q_a}\right)^2 = 1
\end{align}

Because $U > \max_a \pi_a^\star$ the derivative of the left side with respect to $U$ is positive for all $Q_a\in[0,(1+\supFsZero)/(1-\gamma)]^{|\actions|}$. Using the inverse function theorem we get that $U$ is differentiable with respect to Q and that $\pi_a^\star = \left(\frac{\lambda/2}{U - Q_a}\right)^2$ is also differentiable with respect to Q. Finally because $[0,(1+\supFsZero)/(1-\gamma)]^{|\actions|}$ is compact we can conclude that $F$ is $L$-smooth for some $L\geq0$ verifying point 4 of Section~\ref{generalizations}.

\section{Experimental validation of the theoretical results}

In this section, we present the experiments we made to verify the correctness of our sample complexity bounds (Theorem \ref{thm:sample-complexity}) and of our consistency results (Theorem \ref{thm:consistency}).

\subsection{Checking the sample complexity guarantee}

The key step for proving Theorem \ref{thm:sample-complexity} is using Lemma \ref{lemma:sample-complexity}, that bounds the number of calls to the generative model made by a call to $\estV(s, \epsilon)$. 

Figure \ref{fig:nsamplev-check} shows the simulated number of calls to the generative model made by $\nsamplev(\epsilon, \algdelta)$  as a function of $1/\epsilon$ and compares it to our theoretical bound in Lemma \ref{lemma:sample-complexity} and to the number of calls that would be required by a Sparse Sampling strategy, which corresponds to the bound in Proposition \ref{prop:sample-complexity-large-epsilon} extrapolated to all values of $\epsilon$. The simulated values where obtained by computing the following recurrence for several values of $\epsilon$:

\begin{align*}
    \nsamplev^{\mathrm{sim}}(\epsilon,\algdelta)  = \begin{cases}
    1 + \nsamplev^{\mathrm{sim}}\pa{\frac{\epsilon}{\sqrt{\gamma}}, \algdelta} + K N(\sqrt{\barepsilon\epsilon}) \nsamplev^{\mathrm{sim}}\pa{ \sqrt{\frac{\barepsilon\epsilon}{\gamma}} , \algdelta}, \mbox{ if } \epsilon < \barepsilon, \\
    \gamma^{\frac{1}{2}H(\epsilon)(H(\epsilon)-1)}\pa{\frac{2\alpha(\algdelta)}{\epsilon^2}}^{H(\epsilon)}, \mbox{ otherwise.}
    \end{cases}
\end{align*}

Figure \ref{fig:nsamplev-lambda} shows the mumber of calls to the generative model made by $\nsamplev$ as a function of the regularization parameter $\lambda$ in order to achieve a relative error of $0.01$\footnote{ We set $\epsilon = 0.01 V_{\lambda}^{\mathrm{max}}$, where $V_{\lambda}^{\mathrm{max}} = (1+\lambda\log K)/(1-\gamma)$ is an upper bound on the regularized value function.} and its ratio with respect to the number of calls that would be required by Sparse Sampling in the same setting. We see that fewer samples are required as the regularization increases. We also see that, for small $\lambda$, there is no advantage with respect to Sparse Sampling, but $\ouralgo$ has a very large advantage when the regularization $\lambda$ grows.

\begin{figure}[ht!]
    \centering
    \includegraphics[width=0.9\textwidth]{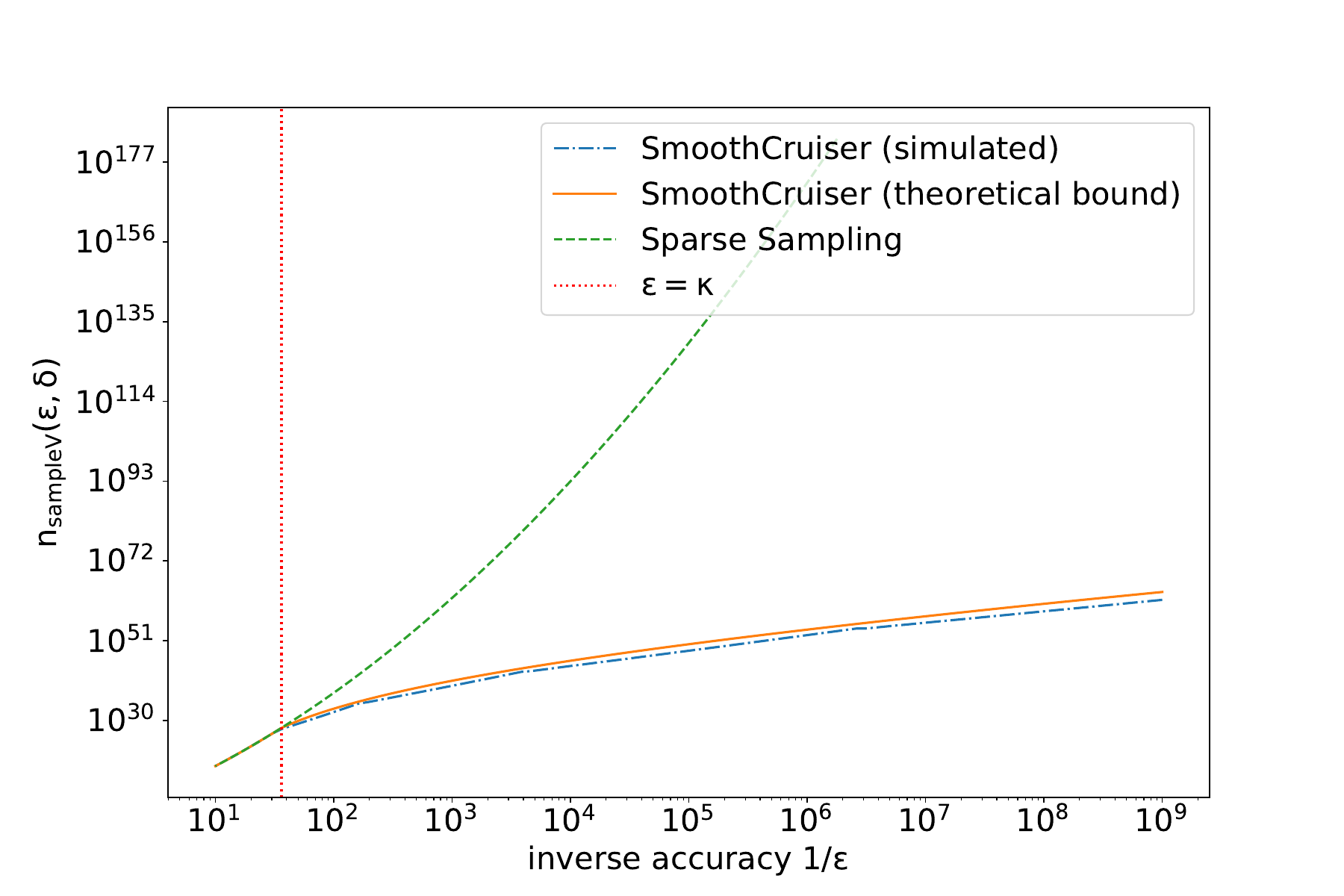}
    \caption{Simulated number of calls to the generative model made by $\nsamplev(\epsilon, \algdelta)$ as a function of $1/\epsilon$ compared to our theoretical bound (Lemma \ref{lemma:sample-complexity}) and to the number of calls that would be required by a Sparse Sampling strategy. The parameters used were: $\gamma=0.2$, $\algdelta=0.1$, $K=2$ and $\lambda=0.1$.}
    \label{fig:nsamplev-check}
\end{figure}

\begin{figure*}[ht!]
    \centering
    \begin{subfigure}[t]{0.49\textwidth}
        \centering
        \includegraphics[width=\textwidth]{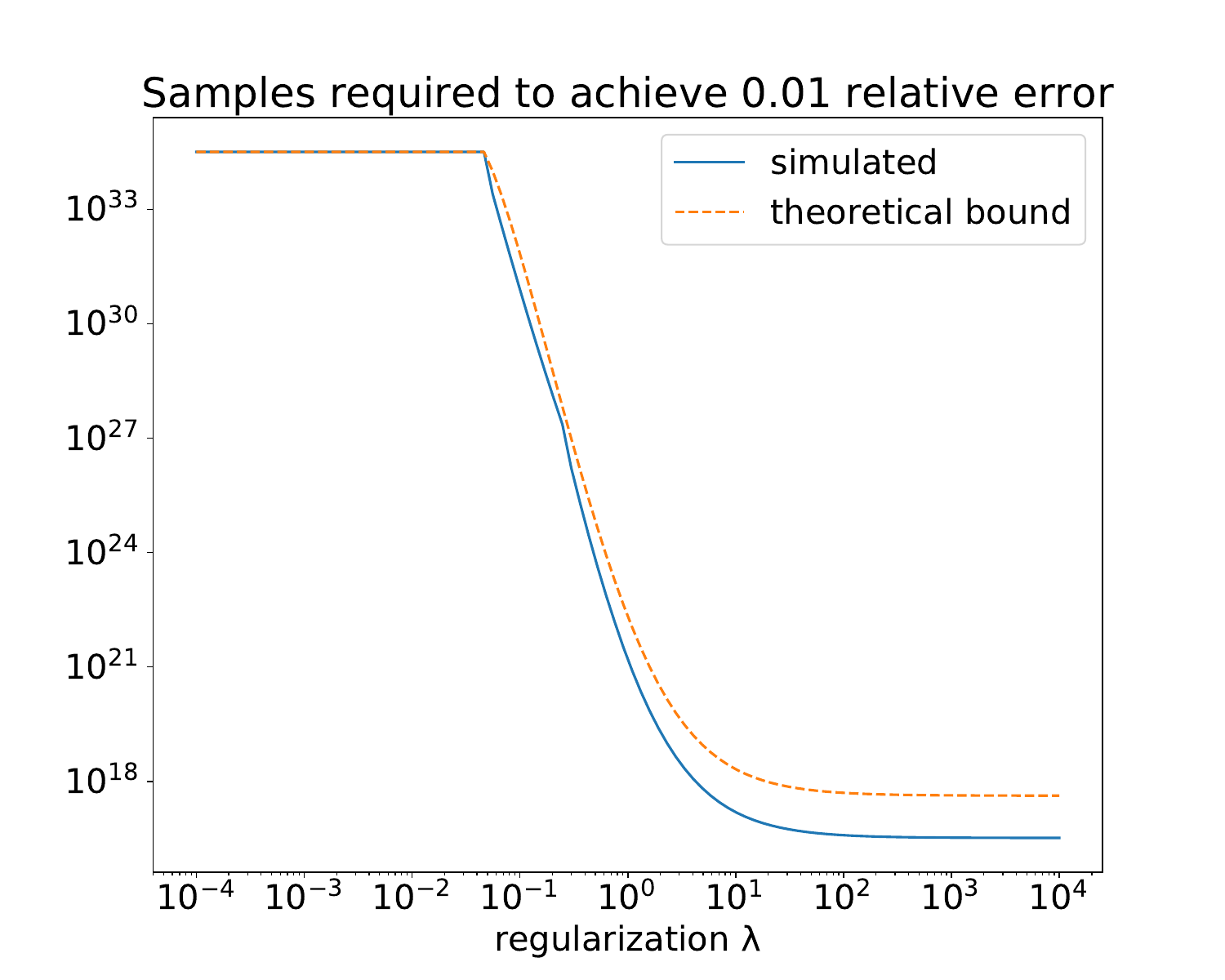}
    \end{subfigure}%
    ~ 
    \begin{subfigure}[t]{0.49\textwidth}
        \centering
        \includegraphics[width=\textwidth]{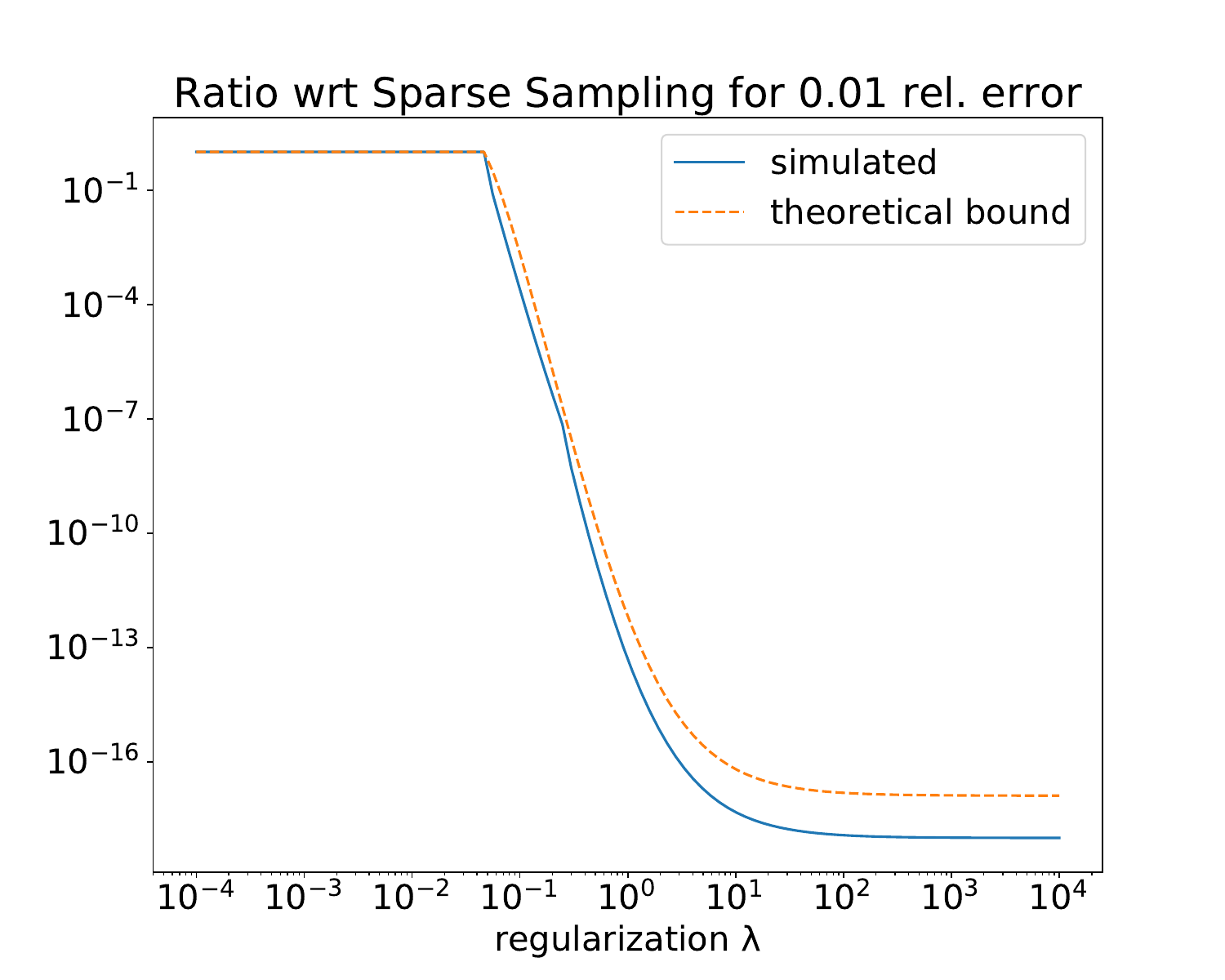}
    \end{subfigure}
    \caption{Number of calls to the generative model made by $\nsamplev$ as a function of the regularization parameter $\lambda$ in order to achieve a relative error of $0.01$ (left) and its ratio with respect to the number of calls that would be required by Sparse Sampling in the same setting (right). The parameters used were: $\gamma=0.2$, $\algdelta=0.1$ and $K=2$.}
    \label{fig:nsamplev-lambda}
\end{figure*}

\subsection{Checking the consistency guarantee}

Using our MCTS analogy in Section \ref{sec:comparison-to-mcts}, the two most computationally costly operations of $\ouralgo$ are the $\selectaction$ and the $\evalleaf$ functions. They both rely on estimates of the $Q$ function with some required accuracy. Hence, for a \textit{sanity-check}, we implemented the function $\estV$ by replacing its calls to $\estQ$($s$, accuracy) by the true Q function at state $s$ plus some accuracy-dependent noise, and we denote this simplified version of $\estV$ by $\estV_{\mathrm{check}}$ . This allowed us to verify that our bounds for the bias of the $\estV$ outputs (Lemma \ref{lemma:consistency}) are correct. After $N_{\mathrm{sim}}$ calls to $\estV_{\mathrm{check}}(\epsilon, \algdelta)$, we compute the error

\begin{align}
    \hat{\Delta}(s, \epsilon) = \frac{1}{N_{\mathrm{sim}}} \sum_{i=1}^{N_{\mathrm{sim}}} \pa{\hat{V}_i(s, \epsilon) - V(s)}
\end{align}
where $s$ is a reference state and $\hat{V}_i(s, \epsilon)$ is the output of the $i$-th call to $\estV_{\mathrm{check}}(s, \epsilon)$. Lemma \ref{lemma:consistency} states that, for some high probability event $B$, we have $-\epsilon \leq \EE{\hat{\Delta}(s, \epsilon) | B} \leq \epsilon$. Hence, for large $N_{\mathrm{sim}}$, we should have $-\epsilon \leq \hat{\Delta}(s, \epsilon) \leq \epsilon$ approximately. 

Table \ref{tab:consistency-check} shows simulated values of $\hat{\Delta}(s, \epsilon)$ and their standard deviations for different environments. The value of $N_{\mathrm{sim}}$ was chosen so that $\hat{\Delta}(s, \epsilon)$ is close to its mean, by using Hoeffdings's inequality and assuming that $\hat{V}_i(s, \epsilon)$ is bounded by $C_\gamma$ (which holds with high probability, by Lemma \ref{lemma:consistency}).

\begin{table}[ht!]
\centering
\begin{tabular}{@{}ll@{}}
\toprule
Environment               & $\hat{\Delta}(s, \epsilon)$      \\ \midrule
5-Chain         & $(-1.21\pm 1.65)\times10^{-2}$ \\
10-Chain        & $(-1.20\pm 1.63)\times10^{-2}$  \\
5x5-GridWorld   & $(-0.71\pm 2.04)\times10^{-2}$   \\
10x10-GridWorld & $(-0.71\pm 2.03)\times10^{-2}$   \\ \bottomrule
\end{tabular}
\vspace{1em}
\caption{Simulated values of $\hat{\Delta}(s, \epsilon)$ and its standard deviation for different environments, for $\epsilon = 0.35$. The value of $\epsilon$ was chosen such that $\epsilon \leq \kappa/4$ in all environments. The parameters used were: $N_{\mathrm{sim}} = 32723$, $\gamma=0.2$ and $\lambda = 10$. The $n$-Chain environments have $K=2$ and $n$ states and the $n\times n$-GridWorld environments have $K=4$ and $n^2$ states.}
\label{tab:consistency-check}
\end{table}


The code for the  experiments is at \url{https://github.com/omardrwch/smoothcruiser-check}.

\end{document}